\documentclass{article}

\PassOptionsToPackage{sort&compress}{natbib}
\usepackage[preprint]{corl_2026} 
\usepackage{graphicx}
\usepackage{booktabs}
\usepackage{amsmath}
\usepackage{amsthm}
\usepackage{amsfonts}
\usepackage{bm}
\usepackage{float}
\usepackage{enumitem}
\usepackage{fontawesome5}
\usepackage[most]{tcolorbox}
\usepackage{multirow}

\newtheorem{theorem}{Theorem}

\newcommand{\method}{{SLS$^2$}}
\newtcbox{\linkpill}{
  on line,
  box align=base,
  colback=gray!10,
  colframe=gray!10,
  arc=1.5mm,
  boxrule=0pt,
  left=2pt,
  right=2pt,
  top=1pt,
  bottom=1pt
}
\usepackage{wrapfig}
\definecolor{slsorange}{HTML}{D97706}
\title{\looseness-1Pixels to Proofs: Probabilistically-Safe Latent World Model Control via Parallel Conformal Robust MPC}

\author{
  Devesh Nath$^\star$, Anutam Srinivasan$^\star$, Haoran Yin$^\star$, Ruitong Jiang, Jeffrey Fang, Glen Chou\\
  Georgia Institute of Technology\\
  \texttt{\{dnath7, asrinivasan350, hyin95, rjiang77, jfang301, chou\}@gatech.edu} \\
  $^\star$Equal contribution (alphabetical ordering)\\
\linkpill{\href{https://trustworthyrobotics.github.io/SLS-squared/}{\textcolor{slsorange}{\faGlobe\ Website}}}
\quad
\linkpill{\href{https://github.com/trustworthyrobotics/SLS-squared}{\textcolor{slsorange}{\faGithub\ Code}}}
\quad
\linkpill{\href{https://youtu.be/3sYYNSQqwSQ}{\textcolor{slsorange}{\faYoutube\ Video}}}
}

\begin{document}
\maketitle


\begin{abstract}
    We present \method, a framework for safe feedback motion planning from pixels using robust model predictive control (MPC) in learned latent world models. Our approach trains an action-conditioned joint-embedding world model with compact Markovian latent states, enabling efficient gradient-based trajectory optimization through learned latent dynamics. To enforce safety for the true system despite imperfect latent predictions, we inform a GPU-accelerated system level synthesis (SLS) robust MPC scheme with conformal prediction to obtain calibrated latent error bounds and robust latent-space constraint sets. We further learn and conformalize a latent constraint checker, allowing the SLS planner to impose probabilistic safety constraints during closed-loop execution. We evaluate our method on vision-based control tasks, where it improves both goal-reaching performance and safety over latent world-model and safe-planning baselines.
\end{abstract}

\keywords{robust visuomotor control, safety verification, JEPA world models} 

\begin{figure}[H]
    \centering\vspace{-8pt}
    \includegraphics[width=\linewidth]{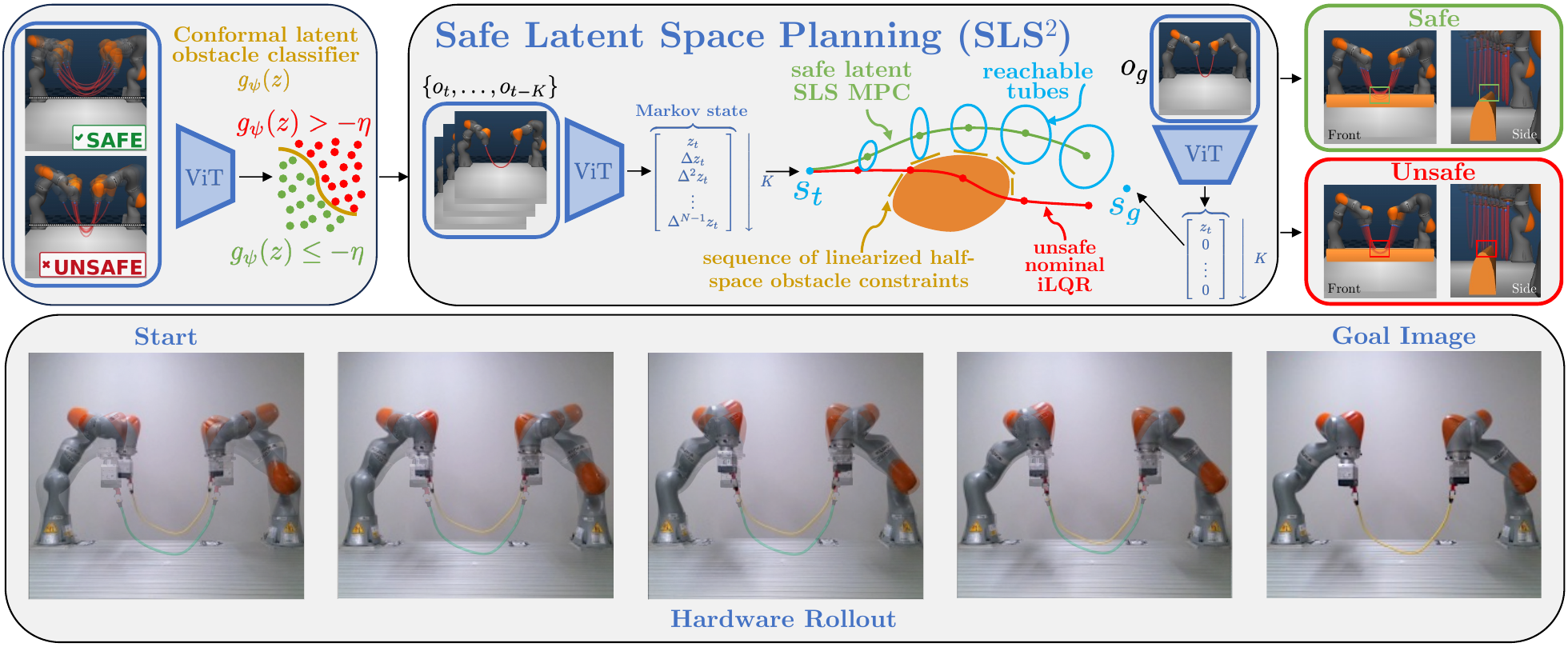}
    \vspace{-1.5em}
    \caption{\textbf{Method overview.} \method~enables robustly safe, reachability-constrained latent planning directly from pixels. \textit{Top row}: We summarize the pipeline for our robust planner. \textit{Bottom row}: We show a sample rollout generated by executing control actions computed by the nominal gradient-based planner on which the robust planner is built. We overlay the goal pose, shown with a green rope, on the rollout to highlight the system’s progress toward the goal state.}
    \label{fig:block_diag}
\end{figure}

\vspace{-18pt}
\section{Introduction}
\vspace{-5pt}

\looseness-1Safe control is a central challenge for reliable open-world robot deployment. Most safety-critical control methods \cite{ames2019controlbarrier, mitchell2005timedependent, manchester2017control} rely on a well-defined state representation on which planning, control, and verification can be performed. However, in the real world, obtaining such a representation is itself nontrivial. This is especially true for deformable objects and tasks where the state is only implicitly defined through perception. Even when a state representation is known, estimating the state from raw observations, such as images, can require extensive offline training data with labeled correspondences between visual inputs and ground-truth states -- data that is often impractical to obtain \cite{chou2022safe, leeman2026vision, chou2023synthesizing, dean2021guaranteeing}.

\looseness-1Recent advances in representation learning and world modeling \cite{ha2018worldmodels, hafner2020dreamer} offer a promising alternative. Learned world models can infer latent state representations directly from observations and learn dynamics in these latent spaces, enabling prediction and planning without requiring manually specified state variables. These models have shown considerable potential for robotic decision-making \cite{zhou2024dinowm, hansen2024tdmpc2}, particularly in settings where explicit state estimation is difficult. However, ensuring that these latent space control policies yield safe behavior under the true system dynamics is difficult. Latent representations are often high-dimensional, their learned dynamics may be inaccurate, and task constraints expressed in the physical or observation space must be translated into the latent space. These issues complicate the use of classical safety-critical control tools in the latent space, which typically assume low-dimensional state spaces and known, calibrated disturbance bounds. Recent work has explored deep learning for approximating latent safety certificates \cite{castaneda2023distribution, nakamura2025generalizing}; however, these methods often rely on heuristic approximations and do not provide sound, calibrated reachability guarantees.

To exploit the advances of world modeling while maintaining rigorous safety assurances, we propose SLS$^2$ (\underline{S}afe \underline{L}atent \underline{S}pace \underline{S}ystem \underline{L}evel \underline{S}ynthesis), a scalable robust model predictive control (MPC) framework for safe planning in learned latent spaces based on the SLS framework \cite{anderson2019system}. Our approach introduces training objectives and optimization strategies that improve gradient-based trajectory optimization through learned latent dynamics. To be robust to prediction error in the latent dynamics, we plan with these learned models using scalable reachability-constrained robust MPC based on SLS. This is made computationally feasible through GPU parallelization, which enables efficient trajectory optimization coupled with local uncertainty propagation around the optimized plan. To enforce safety constraints, we train a constraint checker in the latent space and incorporate it within robust MPC. To enable probabilistic constraint satisfaction, we use conformal prediction (CP) to calibrate bounds on both latent dynamics error and learned constraint-checker error. This ensures that the reachable sets computed by SLS contain the true reachable latent states with high probability, enabling robust safety despite these errors. Specifically, our contributions are: 
\begin{itemize}[leftmargin=1.1em]
    \item We learn compact Markovian latent world models from pixels that are accurate under multi-step rollout and amenable to efficient gradient-based trajectory optimization.
    \item We conformalize latent dynamics errors and in-domain latent support, producing calibrated uncertainty sets for robust planning. We further learn and conformalize latent safety classifiers, enabling task constraints to be enforced directly during latent space planning.
    \item We show how to use CP-calibrated uncertainty sets in GPU-accelerated SLS to scalably perform reachability-constrained robust MPC in high-dimensional learned latent spaces.
    \item We validate nominal trajectory optimization and robust MPC with our learned models on both simulated and hardware vision-based control tasks, spanning navigation, pick-and-place, and deformable object manipulation. Our method improves safety and goal-reaching performance over world-model and safe-planning baselines.
\end{itemize}

\section{Related Work}

\paragraph{World Model}
\looseness-1World models learn predictive representations of dynamics, enabling agents to reason about future outcomes through imagined rollouts~\cite{ha2018worldmodels,hafner2019planet,hafner2020dreamer,hafner2021dreamerv2,hafner2023dreamerv3,hafner2025dreamerv4}. Early latent world models paired representation learning with reconstruction-based prediction and policy learning in imagination~\cite{hafner2019planet,hafner2020dreamer,hafner2021dreamerv2}, while TD-MPC and TD-MPC2 combine latent dynamics with online model-predictive control for continuous control~\cite{hansen2022tdmpc,hansen2024tdmpc2}. More recently, joint-embedding predictive architectures (JEPAs) have emerged as a promising approach to representation-space world modeling by predicting future embeddings rather than pixels~\cite{lecun2022path,assran2023ijepa,bardes2024vjepa,assran2025vjepa2}. In robotics and offline goal-conditioned planning, PLDM learns reward-free latent dynamics for test-time planning~\cite{sobal2025pldm}, while DINO-WM predicts future DINOv2 visual features from offline action-conditioned trajectories for zero-shot planning without reward supervision~\cite{oquab2024dinov2,zhou2024dinowm}. Large-scale generative world models and interactive simulators further demonstrate the promise of learned predictive models for controllable environments~\cite{micheli2023iris,alonso2024diamond,micheli2024deltairis,bruce2024genie}. Our work builds on LeWorldModel (Le-WM), an end-to-end JEPA world model that learns from raw pixels via next-embedding prediction with Gaussian embedding regularization and without explicitly predicting future images, then performs latent-space MPC for goal-reaching control~\cite{lewm}. We introduce training modifications that improve both predictive accuracy and gradient-based planning.

\vspace{-4pt}
\paragraph{Safe Latent Planning} 
Safety filters enforce constraints by minimally modifying candidate actions from a nominal controller. Classical approaches such as CBF-QPs, HJ reachability, and predictive safety filters offer principled safety guarantees but typically require known state representations, known dynamics, and hand-designed safe or failure sets~\cite{ames2017controlbarrier,ames2019controlbarrier,mitchell2005timedependent,fisac2019general,wabersich2021predictive,wabersich2023datadrivensafetyfilters}. Recent data-driven methods reduce this burden by learning barrier functions, output-feedback certificates, distributional safety constraints from demonstrations and offline data, or by performing uncertainty-aware model-based control with learned dynamics ~\cite{robey2020learningcbfs,lindemann2024robustoutputcbfs,kang2022lyapunovdensity,faroni2025uncertainty,cao2025optimistic,chou2021model,knuth2021planning,knuth2023statistical,suh2023fighting,dodeja2025accelerating,chua2018deep}. Closest to our setting, safe latent-space methods extend safety filtering to high-dimensional visuomotor control. In-distribution barrier functions learn CBF-like filters in latent space from safe demonstrations~\cite{castaneda2023distribution}, while latent safety filters approximate HJ reachability in learned world-model latents to reason about hard-to-specify failures~\cite{nakamura2025generalizing}. Follow-up work improves latent safety filtering with uncertainty-aware OOD detection and conformal calibration~\cite{seo2025unisafe}, smooth latent CBFs~\cite{nakamura2025latentcbf}, and latent barrier certificates learned jointly with world models~\cite{anand2025safetycertification}. Latent Policy Barrier treats the expert demonstration manifold as an implicit safety barrier and corrects actions at inference time using learned latent dynamics~\cite{sun2025latentpolicybarrierlearning}. These methods show that latent representations can enable safety reasoning from images. However, most existing methods operate as policy filters without formal guarantees, safety predictors that do not perform control \cite{mao2024safe}, or in-distribution recovery mechanisms around a \textit{given} base policy. In contrast, we directly synthesize safety-constrained feedback policies and motion plans through reachability-informed latent space control, while providing calibrated probabilistic guarantees via SLS and conformal prediction.

\begin{figure*}[t] 
    \centering
    \includegraphics[width=\linewidth, trim={7.6cm 9cm 7.5cm 9cm}, clip]{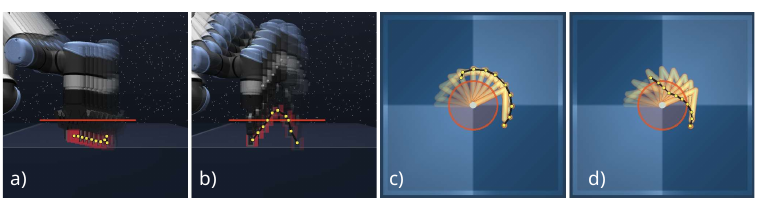}
    \vspace{-23pt}
    \caption{\textbf{Robust planning}. \textbf{Cube (a-b)} \textbf{and} \textbf{Reacher (c-d)}. \method \ maps pixels directly to robustly-safe actions via latent robust MPC, ensuring robust constraint satisfaction (red) compared to unconstrained baselines. \textbf{a, b)} The Cube task with and without \method \ enforcing the height threshold. \textbf{c, d)} The Reacher task with (c) and without (d) \method \ enforcing the angle constraint on the second joint.\vspace{-9pt}}
    \label{fig:cube_and_reacher}
\end{figure*}

\vspace{-5pt}
\section{Preliminaries and Problem Statement}
\vspace{-5pt}

We consider an unknown partially observed dynamical system described by an \textit{unknown} underlying state representation 
$x_k\in\mathcal{X}$ and control $u_k\in\mathcal{U}\subseteq\mathbb{R}^{n_u}$, and observations $o_k \in \mathcal{O}\subseteq \mathbb{R}^{n_o}$:
\begin{equation}\label{eq:true_dyn}
    x_{k+1}=g(x_k,u_k),
    \qquad
    o_k=h(x_k),
\end{equation}
where $o_k\in\mathcal{O}$ are RGB images. 
In our setting, neither the true state $x_k$ nor its
state representation $\mathcal{X}$ is available to the planner. $g$ and $h$ are also unknown; we only observe image measurements $o \in \mathcal{O}$. 

\paragraph{Problem Statement.}
We consider safe feedback motion planning from image observations using a learned latent world
model. We are given an offline dataset of image transitions
$\mathcal{D}:=\{(o_t,u_t,o_{t+1})\}$ collected from task-relevant rollouts. The true
system evolves in an unknown state space $\mathcal{X}$, while the planner only observes images $o \in \mathcal{O}$ from \eqref{eq:true_dyn}. We wish to learn a compact latent representation where prediction, planning, and
reachability-based safety verification are tractable.

\looseness-1In this paper, we train a visual encoder $\mathrm{enc}_\theta:\mathcal{O}\to\mathbb{R}^{n_z}$, with
$z_t=\mathrm{enc}_\theta(o_t)$. Using $\mathrm{enc}_\theta$, we construct a Markovian latent state
$s_t^z=[z_t,\Delta z_t,\ldots,\Delta^Kz_t]\in\mathcal{S}\subseteq\mathbb{R}^{n_s}$,
where $\Delta z_t=z_t-z_{t-1}$, from a history of encoded images. The learned latent dynamics
$f_\phi:\mathcal{S}\times\mathcal{U}\to\mathcal{S}$ predict the next Markovian state 
$\hat s_{t+1}^z=f_\phi(s_t^z,u_t)$. As $f_\phi$ can be inaccurate, we
model its prediction error as
$s_{t+1}^{z}=f_\phi(s_t^z,u_t)+e_t$, where $\mathcal{E} \subseteq \mathbb{R}^{n_s}$ is calibrated from
a held-out subset of $\mathcal{D}$ using CP to contain the true prediction error $e_t$ with a user-specified probability $1-\delta \in (0, 1)$. Specifically, we solve: 

\noindent\textbf{Problem 1.}
\textit{Learning a planning- and verification-friendly latent world model.}
Using $\mathcal{D}$, learn $\mathrm{enc}_\theta$ and $f_\phi$ such that the latent
state $s_t^z$ is predictive under actions, the learned dynamics are accurate under
multi-step rollout, differentiable, and compatible with gradient-based trajectory
optimization. Construct uncertainty sets $\mathcal{E}(s,u) \subseteq \mathbb{R}^{n_s}$ such that the prediction error $e$ at a particular $(s,u)$ satisfies $e \in \mathcal{E}(s,u)$ with probability $1-\delta$. Finally, learn a latent safety score $c_\psi:\mathbb{R}^{n_z}\to\mathbb{R}$, where larger values indicate safer states, and define the constraint score $g_\psi(s):=-c_\psi(\Pi_z s)$, where $\Pi_z$ extracts the zeroth-order latent embedding.

\noindent\textbf{Problem 2.}
\textit{Reachability-informed robust MPC in latent-space.}
Given $f_\phi$, $\mathcal{E}$, and a 
latent safe set
$\mathcal{S}_{\mathrm{safe}}:=\{s\in\mathcal{S}\mid g_\psi(s)\le-\eta\}$ for $\eta \in \mathbb{R}$,
optimize a nominal length-$T$ latent trajectory
$\bm{\zeta}:=\{\zeta_k\}_{k=0}^{T}$ and controls
$\mathbf{v}:=\{v_k\}_{k=0}^{T-1}$ satisfying $\zeta_{k+1}=f_\phi(\zeta_k,v_k)$, together with
a causal state-feedback controller
$\pi:=(\pi_0,\ldots,\pi_{T-1})$, $\pi_k:\mathcal{S}^{k+1}\to\mathcal{U}$, that stabilizes
the true latent dynamics
$s_{k+1}^{z}=f_\phi(s_k^z,\pi_k(s_{0:k}^z))+e_k$ about
$(\bm{\zeta},\mathbf{v})$, where $e_k$ is the true prediction error realized by executing on the true system \eqref{eq:true_dyn}. This MPC policy should produce reachable tubes
$\mathcal{R}_k^s$ and $\mathcal{R}_k^u$ such that, with probability at least
$1-\delta$,
$s_k^{z}\in\mathcal{R}_k^s\subseteq\mathcal{S}_{\mathrm{safe}}$ for all
$k=0,\ldots,T$ and
$\pi_k(s_{0:k}^{z})\in\mathcal{R}_k^u\subseteq\mathcal{U}$ for all
$k=0,\ldots,T-1$.

\section{Methodology}
\label{sec:methodology}
\looseness-1Our method, \method~(Fig. \ref{fig:block_diag}), trains a Markovian latent model and constraint classifier (Sec. \ref{sec:method_training}), conformalizes its uncertainty (Sec. \ref{sec:cp}), and incorporates the resulting calibrated uncertainty sets into SLS-based robust MPC for probabilistically-safe feedback motion planning (Sec. \ref{sec:sls_mpc}).

\subsection{Training World Models Friendly for Gradient-Based Robust MPC}\label{sec:method_training}

We use the given offline dataset $\mathcal{D}$ to train an action-conditioned latent world model following a Le-WM-style \cite{lewm} joint-embedding predictive architecture. Given an image observation $o_t$, an encoder $\mathrm{enc}_{\theta}$ maps the image into a compact latent embedding
$z_t = \mathrm{enc}_{\theta}(o_t)$. 
The encoder is implemented as a ViT backbone trained from scratch, followed by an MLP projector that maps the encoder output to the latent embedding space. Before encoding, image observations are resized to a fixed resolution, normalized, and organized into short temporal windows sampled from offline trajectories. Actions are normalized using the dataset mean and standard deviation. 
To model dynamics, we construct a Markov latent state by augmenting the current embedding with finite-difference latent features:
\begin{equation}
    s_t^z =
    \left[
        z_t,\;
        \Delta z_t,\;
        \ldots,\;
        \Delta^K z_t
    \right],
    \qquad
    \Delta z_t = z_t - z_{t-1},
\end{equation}
where $K$ denotes the Markov order. The latent dynamics predictor $f_{\phi}$ is an MLP that takes the current Markov latent state and action as input and predicts the next Markov latent state:
\begin{equation}\label{eq:latent_dynamics}
    \hat{s}_{t+1}^z = f_{\phi}(s_t^z, a_t).
\end{equation}
For multi-step training, the predictor is rolled out autoregressively over a horizon $H$, where each predicted latent state is fed back into the model for the next prediction step $\hat{s}_{t+h+1}^z = f_{\phi}(\hat{s}_{t+h}^z, a_{t+h})$ for $h = 0,\ldots,H-1.$ 
The main training loss is a multi-step latent prediction objective. At each prediction step, the predicted Markov latent state is compared against the target Markov latent state computed from the encoded ground-truth future observations:
\begin{equation}
    \textstyle\mathcal{L}_{\mathrm{pred}}
    =
    \frac{1}{H}
    \sum_{h=1}^{H}
    \left\|
        \hat{s}_{t+h}^z - s_{t+h}^z
    \right\|_2^2 .
\end{equation}
We follow Le-WM \cite{lewm} and prevent representational collapse via the SIGReg loss \cite{balestriero2025lejepa} and additionally encourage consecutive latent velocity vectors to align via temporal straightening \cite{wang2026straightening}:
\noindent
\begin{minipage}{0.35\textwidth}
    \begin{equation}\label{eq:sigreg}
        \mathcal{L}_{\mathrm{sig}}
        =
        \mathrm{SIGReg}\left(\{z_t\}_{t=1}^{T}\right),
    \end{equation}
\end{minipage}%
\hfill
\begin{minipage}{0.6\textwidth}
    \begin{equation}
        \textstyle\mathcal{L}_{\mathrm{str}}
        =
        \frac{1}{T-2}
        \sum_{t=2}^{T-1}
        \left(
            1 -
            \frac{
                \left(z_t - z_{t-1}\right)^{\top}
                \left(z_{t+1} - z_t\right)
            }{
                \left\|z_t - z_{t-1}\right\|_2
                \left\|z_{t+1} - z_t\right\|_2
            }
        \right).
    \end{equation}
\end{minipage}

\noindent
The full training objective is therefore
$\mathcal{L}
    =
    \mathcal{L}_{\mathrm{pred}}
    +
    \lambda_{\mathrm{sig}}\mathcal{L}_{\mathrm{sig}}
    +
    \lambda_{\mathrm{str}}\mathcal{L}_{\mathrm{str}}$.

\paragraph{Obstacle Classifier.}
For constrained tasks, we train a latent obstacle classifier to detect whether a predicted latent state violates task constraints. For each task, we build a balanced dataset of obstacle and non-obstacle observations by sampling task-relevant configurations, labeling them with the known simulator constraint, and rendering them as RGB images. This yields positive and negative constraint examples without relying on rollout failures. While the labeling rule is task-specific, the dataset has a common form across tasks,
$\mathcal{D}_{\mathrm{obs}}
    =
    \{(o_i, y_i)\}_{i=1}^{N}$, with 
    $y_i \in \{0,1\},$
where $y_i=0$ denotes a constraint-violating image and $y_i=1$ denotes a non-violating image.

We train the classifier in the latent space rather than in pixel space. Using
the frozen encoder, each image is mapped to
$z_i=\mathrm{enc}_{\theta}(o_i)$. A small MLP classifier $c_{\psi}:\mathbb{R}^{n_z}\to\mathbb{R}$ maps
$z_i$ to a signed score $c_i=c_{\psi}(z_i)$, where positive scores indicate non-violating states and negative scores indicate constraint-violating states. Thus, binary labels
$\{0,1\}$ are mapped to the signed labels $\{-1,+1\}$, with $y_i=0$ mapped to $-1$ and $y_i=1$ mapped to $+1$, and
the classifier is trained with a hinge loss,
\begin{equation}
    \textstyle\mathcal{L}_{\mathrm{obs}}
    =
    \frac{1}{N}
    \sum_{i=1}^{N}
    \max\left(0,\; m - \tilde{y}_i c_{\psi}(z_i)\right),
\end{equation}
\looseness-1where $m$ is the classification margin. We normalize latent features using the
training-split mean and standard deviation, then calibrate the classifier threshold on
a held-out split. Denote $\mathcal{I}_{\mathrm{cal}}^{-}:=\{i\in\mathcal{I}_{\mathrm{cal}}:y_i=0\}$ as the violating calibration
indices. For each $i\in\mathcal{I}_{\mathrm{cal}}^{-}$, we compute the nonconformity
score
$r_i = \max\left(0,\; c_{\psi}(z_i)\right)$. 
We then choose a conformal threshold $\eta$ as the empirical $(1-\delta)$ quantile of these nonconformity scores. At planning time, a latent Markov state $s$ is treated as safe if
\begin{equation}
    c_{\psi}(\Pi_z s) \ge \eta,\qquad \Longleftrightarrow\qquad g_{\psi}(s):=-c_{\psi}(\Pi_z s) \le -\eta.
\end{equation}
This yields a conservative obstacle classifier used during constrained planning: the
world model predicts latent rollouts, and the classifier detects whether they enter the
learned obstacle set.

\subsection{Latent Uncertainty Quantification via Conformal Prediction (CP)}\label{sec:cp}

To ensure closed-loop safety for planning, we need two additional sources of uncertainty to be bounded: 1) latent dynamics prediction error and 2) in-distribution domain calibration. First, to quantify latent prediction error, we use split CP \cite{tibshirani2019conformal, lei2014distribution}. 
Let $\mathcal{D}_{\mathrm{cal}}=\{(s_i,u_i,s_i^+)\}_{i=1}^n$ be a held-out set of
latent transitions obtained by encoding a subset of $\mathcal{D}$, and define residuals $e_i:=s_i^+-f_\phi(s_i,u_i)$. To capture
state-action-dependent error geometry, we train an uncertainty model
$\Sigma_\psi:\mathcal{S}\times\mathcal{U}\to\mathbb{S}_{++}^{n_s}$ using the
multivariate Gaussian negative log-likelihood loss
\begin{equation}\label{eq:MGNLL}
    \mathcal{L}_{\mathrm{MGNLL}}
    =
    \textstyle\frac{1}{2}
    \left(
        e_i^\top \Sigma_\psi(s_i,u_i)^{-1} e_i
        +
        \log\det \Sigma_\psi(s_i,u_i)
    \right).
\end{equation}
The covariance model is not assumed to be calibrated; it only provides a local ellipsoid
shape for the prediction error. We calibrate its scale using CP with scores
$r_i:=e_i^\top\Sigma_\psi(s_i,u_i)^{-1}e_i$. 
Let $r_{(1)}\le\cdots\le r_{(n)}$ be the sorted scores. For trajectory-level failure
probability $\delta$ over horizon $T$, set $\bar\delta:=\delta/T$ and
$q:=r_{\left(\left\lceil (n+1)(1-\bar\delta)\right\rceil\right)}$, 
with $q=+\infty$ if the index exceeds $n$. For any latent state-action pair $(s,u)$,
the calibrated one-step error set is
\begin{equation}\label{eq:cp_set_prediction}
    \mathcal{E}(s,u)
    :=
    \{e\mid e^\top\Sigma_\psi(s,u)^{-1}e\le q\}.
\end{equation}
Writing $\Sigma_\psi(s,u)=L_\psi(s,u)L_\psi(s,u)^\top$ using its Cholesky factor $L_\psi$, this is equivalently
$\mathcal{E}(s,u)=\sqrt{q}\,L_\psi(s,u)\mathcal{B}^{n_s}$, where $\mathcal{B}^{n_s} := \{s \in \mathbb{R}^{n_s} \mid \Vert s \Vert_2 \le 1\}$. 

Second, we conformalize an in-domain latent region to keep the planner within the
support of the calibration data. Since the SIGReg loss encourages the latent
embeddings to follow an approximately Gaussian distribution
\cite{balestriero2025lejepa}, we define the in-domain set as a calibrated Gaussian
density sublevel set. Concretely, we fit an ellipsoid to the calibration embeddings
and choose its conformal threshold so that it contains most calibration latents with
the desired coverage. Let $\mu$ and $\Sigma_{\mathrm{ID}}$ be the empirical
mean and covariance of calibration latents in $\mathcal{D}_\textrm{cal}$, and define
$d_i:=(s_i-\mu)^\top\Sigma_{\mathrm{ID}}^{-1}(s_i-\mu)$. With $q_{\mathrm{ID}}$ the
$(1-\alpha_{\mathrm{ID}})$ split conformal quantile of $\{d_i\}$, the calibrated
in-domain set is
\begin{equation}\label{eq:in_domain}
    \mathcal{I}:=
    \{s\mid (s-\mu)^\top\Sigma_{\mathrm{ID}}^{-1}(s-\mu)\le q_{\mathrm{ID}}\}.
\end{equation}
During planning, we require $\mathcal{R}_k^s\subseteq\mathcal{I}$ for all $k$. This
controls support mismatch, while the CP guarantee itself assumes that closed-loop
residual scores are exchangeable with the calibration scores. In practice, we support
this assumption by collecting calibration transitions from representative randomized
rollouts that vary initial states and goals. 

\subsection{Conformalized Robust MPC via System Level Synthesis (SLS)}
\label{sec:sls_mpc}

\vspace{-3pt}
To synthesize robust latent feedback policies, we use system
level synthesis (SLS)~\cite{anderson2019system}. SLS enables efficient robust MPC by optimizing
over closed-loop system responses rather than feedback gains directly, and admits fast
GPU-parallel implementations for high-dimensional systems~\cite{fang2026safe}. This
is particularly important in latent-space planning, where the latent state dimension can be high. We build on \cite{srinivasan2026safety}, using the conformal
calibration in Sec.~\ref{sec:cp} to define the disturbance sets to be propagated by SLS. 
We model the learned latent dynamics with calibrated additive uncertainty,
\begin{equation}
    s_{k+1}=f_\phi(s_k,u_k)+E_k w_k,
    \qquad
    w_k\in\mathcal{B}^{n_s},
\end{equation}
where $E_k = \sqrt{q} L_\psi(s_k, u_k)$, recovering $E_k \mathcal{B}^{n_s}$ as the CP-calibrated, ellipsoidal one-step latent prediction-error
set defined in \eqref{eq:cp_set_prediction}. Given the current latent state $\bar s_0$, SLS optimizes a nominal trajectory
$\bm{\zeta}:=\{\zeta_k\}_{k=0}^{T}$, controls
$\mathbf{v}:=\{v_k\}_{k=0}^{T-1}$, and closed-loop response matrices
$\Phi^s,\Phi^u$. The responses map disturbances to deviations from the nominal state
and control trajectories. Around $(\zeta_k,v_k)$, let
$A_k:=\nabla_s f_\phi(\zeta_k,v_k)$ and
$B_k:=\nabla_u f_\phi(\zeta_k,v_k)$. From Sec.~\ref{sec:cp}, if
$\Sigma_\psi(\zeta_k,v_k)=L_\psi(\zeta_k,v_k)L_\psi(\zeta_k,v_k)^\top$, we set
\begin{equation}
    E_k:=\sqrt{q}\,L_\psi(\zeta_k,v_k).
\end{equation}
Thus SLS propagates the CP-calibrated ellipsoidal model errors through the horizon to
construct reachable tubes and constraint tightenings. 
The resulting CP-calibrated SLS MPC problem is
\begin{subequations}
\begin{align}
\min_{\bm{\zeta},\mathbf{v},\Phi^s,\Phi^u}\quad
    & J(\bm{\zeta},\mathbf{v}) + J_f(\zeta_T) + H(\Phi^s,\Phi^u)\label{eq:sls_objective} \\
\mathrm{s.t.}\quad
    & \zeta_{k+1}=f_\phi(\zeta_k,v_k),\quad \zeta_0=\bar s_0,
      && k=0,\ldots,T-1, \label{eq:sls_dyn}\\
    & \Phi^s_{k+1,j}=A_k\Phi^s_{k,j}+B_k\Phi^u_{k,j},
      && 0\le j<k,\ k=0,\ldots,T-1, \label{eq:sls_slp}\\
    & \Phi^s_{j+1,j}=E_j,
      && j=0,\ldots,T-1, \label{eq:sls_slp_init}\\
    & g_i(\zeta_k,v_k)+b_i+
      \tau_{i,k}
      \le 0,
      && i=1,\ldots,n_c,\ k=0,\ldots,T, \label{eq:sls_tightenings}\\
    & \mathcal{R}_k^s\subseteq\mathcal{I},
      && k=0,\ldots,T,\label{eq:sls_indist}
\end{align}
\label{eq:cp_sls_mpc}
\end{subequations}

\vspace{-14pt}
where $\tau_{i,k} := \sum_{j=0}^{k-1}
      \Vert
          \nabla_s g_i(\zeta_k,v_k)\Phi^s_{k,j}
          +
          \nabla_u g_i(\zeta_k,v_k)\Phi^u_{k,j}
      \Vert_2$, and $\mathcal{R}_k^s
    :=
    \zeta_k\oplus\bigoplus_{j=0}^{k-1}\Phi^s_{k,j}\mathcal{B}^{n_s}$, 
    $\mathcal{R}_k^u
    :=
    v_k\oplus\bigoplus_{j=0}^{k-1}\Phi^u_{k,j}\mathcal{B}^{n_s}$
are the state and control reachable tubes, respectively. Here, $\bigoplus$ and $\oplus$ denote Minkowski sums, and $b_i$ is a constraint-index-dependent bias equal to $\eta$ for obstacle constraints and zero for control constraints. For sets $\mathcal{A}$ and $\mathcal{B}$, $\mathcal{A} \oplus \mathcal{B} = \{a + b\mid a \in\mathcal{A}, b\in\mathcal{B}\}$. The objective \eqref{eq:sls_objective} combines nominal task cost,
terminal cost, and a tube-size regularizer $H(\Phi^s, \Phi^u)$ (see App. \ref{app:sls} for details). The constraint \eqref{eq:sls_dyn} enforces nominal
feasibility under the latent dynamics, and \eqref{eq:sls_slp}-\eqref{eq:sls_slp_init} enforce that SLS
correctly propagates disturbance, with the CP-calibrated disturbance set injected through
\eqref{eq:sls_slp_init}. The constraints \eqref{eq:sls_tightenings} are tightened by reachable tube margins $\tau_{i,k}$, enforcing robust satisfaction of the
learned latent safety constraints over the full reachable tube, and
\eqref{eq:sls_indist} keeps the reachable tube inside the conformalized
in-domain latent region \eqref{eq:in_domain}. For instance, with $g_\psi(s):=-c_\psi(\Pi_zs)$, the obstacle constraint $g_\psi(s_k) \le -\eta$ becomes $g_\psi(\zeta_k) +\eta + \tau_k \le 0$ at each step $k$ to enforce \emph{robust} constraint satisfaction. This CP-informed SLS scheme provides the following robustness guarantee (proof in App. \ref{app:proofs}):

\begin{theorem}[Probabilistic containment of conformal SLS tubes]
\label{thm:cp_sls_latent_tube}
Consider the latent closed-loop dynamics
$s_{k+1}=f_\phi(s_k,u_k)+e_k$,
for $k=0,\ldots,T-1$.
Let $q$ be the split conformal quantile of the calibration scores
$r_i
    =
    e_i^\top \Sigma_\psi(s_i,u_i)^{-1}e_i$ and
    $e_i:=s_i^+-f_\phi(s_i,u_i)$, 
computed at per-step miscoverage $\bar\delta=\delta/T$. For each closed-loop
transition, define the test score
$r_k^{\mathrm{test}} := e_k^\top \Sigma_\psi(s_k,u_k)^{-1}e_k$. 
Assume that, for each $k$, $r_k^{\mathrm{test}}$ is exchangeable with the calibration
scores. Let $\Sigma_\psi(s,u)=L_\psi(s,u)L_\psi(s,u)^\top$ and define the local
conformal error set
$\mathcal{E}(s,u)
    :=
    \sqrt{q}\,L_\psi(s,u)\mathcal{B}^{n_s}$ as defined in \eqref{eq:cp_set_prediction}. 
Suppose the SLS controller is synthesized with disturbance matrices $E_k$ such that
the realized conformal error sets are contained in the SLS disturbance sets,
$\mathcal{E}(s_k,u_k)\subseteq E_k\mathcal{B}^{n_s}$, for $k=0,\ldots,T-1$. 
If the SLS controller is robustly correct for all disturbances
$w_k\in\mathcal{B}^{n_s}$ in
$\Delta s_{k+1}
    =
    A_k\Delta s_k+B_k\Delta u_k+E_kw_k$, where $\Delta s_k := s_k - \zeta_k$ and $\Delta u_k := u_k - v_k$, 
then
$\mathbb{P}\left[
        (s_k\in\mathcal{R}_k^s,\;
        \forall k=0,\ldots,T) \wedge\;
        (u_k\in\mathcal{R}_k^u,\;
        \forall k=0,\ldots,T-1)
    \right]
    \ge 1-\delta$.
\end{theorem}

In addition to robust MPC, the latent dynamics in \eqref{eq:latent_dynamics}, obtained via Sec. \ref{sec:method_training}, can also be used for gradient-based nominal trajectory optimization by removing the reachability-related terms in \eqref{eq:cp_sls_mpc} and retaining only the nominal latent space dynamics constraints in \eqref{eq:sls_dyn}.

\vspace{-15pt}
\section{Experiments}\label{sec:experiments}

\vspace{-8pt}
\looseness-1To show the utility of our world model learning method (Sec. \ref{sec:method_training}) for \textit{nominal gradient-based planning} and the safety assurances provided by our \textit{robust MPC planner} (\method, Sec. \ref{sec:sls_mpc}), we evaluate on four tasks: Reacher \cite{tassa2018deepmind}, OGBench Cube \cite{park2025ogbench}, Push-T \cite{zhou2024dinowm}, and a bimanual rope manipulation task. 

We compare our gradient-based nominal planner against world-model planning baselines in Sec. \ref{sec:experiments_nominal}, and compare \method~against safe latent-space planning baselines in Sec. \ref{sec:experiments_safe}. For nominal planning, as discussed in Sec. \ref{sec:sls_mpc}, we use GPU-SLS~\cite{fang2026safe} without tubes or constraints (which simplifies to iLQR \cite{li2004iterative}). For robust planning, we use \method. We compare against three latent-space world-model planning baselines: Le-WM \cite{lewm}, DINO-WM \cite{zhou2024dinowm}, and PLDM \cite{sobal2025pldm}. For safe latent space planning, we compare against two safety-focused baselines: HJ-filtered \cite{nakamura2025generalizing} and LPB \cite{sun2025latentpolicybarrierlearning}, as well as two non-robust ablations of \method~(i.e., constrained iLQR without tubes or constraint \eqref{eq:in_domain}). 

\looseness-1All results are run on an Intel i9-14900K CPU, 64\,GB RAM, with an NVIDIA RTX 4090 GPU. 
For all tasks, planning performance is evaluated on held-out test trajectories. To benchmark the nominal planner, we evaluate over five random seeds, with 40 evaluation episodes per seed. Unlike prior works that evaluate shorter-horizon planning segments, we require the planner to solve the complete task from the initial state to the corresponding terminal goal, final state, or task-specific completion condition. This makes the benchmark substantially more difficult, since errors accumulate over the full task horizon rather than only over a short local segment. For constrained tasks, we evaluate on 35 randomized start and goal conditions per task. Details on implementation and setup are in App. \ref{app:exp_details}. 

We report metrics according to the task setting. For unconstrained tasks, we report the task success rate (higher is better) and the minimum distance to the goal across evaluation trials (lower is better). For constrained tasks, we also report the safety rate (adherence to constraints) and the robust success rate (when both constraint satisfaction and goal-reach conditions are met).

\begin{table}[t]
\centering
\small
\caption{Ablation study on Reacher. The reference setting is the full method described in Sec.~\ref{sec:methodology}. World-model ablations keep the planner fixed and change only the latent model/training component. The planner ablation keeps the world model fixed and changes only the MPC optimizer.}
\label{tab:reacher_ablation_combined}
\resizebox{\linewidth}{!}{
\begin{tabular}{llcccc}
\toprule
Ablation & Variant & \textbf{Success (\%)} $\uparrow$ & Min. qpos dist. $\downarrow$ & Final qpos dist. $\downarrow$ & Solve time / step (ms) $\downarrow$ \\
\midrule
Reference 
& Full model + iLQR & $\mathbf{83.50 \pm 4.06}$ & $0.377 \pm 0.119$ & $0.497 \pm 0.152$ & $426.35 \pm 17.52$ \\
\midrule
\multirow{2}{*}{World model} 
& Non-Markov history-2 + iLQR & $1.00 \pm 1.22$ & $1.786 \pm 0.088$ & $2.080 \pm 0.087$ & $67.00 \pm 3.01$ \\
& Single-step loss + iLQR & $24.50 \pm 4.58$ & $1.283 \pm 0.138$ & $1.746 \pm 0.205$ & $578.33 \pm 16.56$ \\
\midrule
Planner 
& Full model + CEM & $57.00 \pm 1.87$ & $0.380 \pm 0.100$ & $0.511 \pm 0.106$ & $67.00 \pm 0.94$ \\
\bottomrule
\end{tabular}
}
\vspace{-10pt}
\end{table}

\begin{table}[t]
\centering
\small
\caption{Nominal planning performance across benchmark tasks. Success is reported in percentage, and Dist. denotes minimum distance to the goal.}
\label{tab:nominal_results}
\resizebox{\linewidth}{!}{
\begin{tabular}{llcc@{\hskip 15pt}llcc}
\toprule
Task & Method & \textbf{Success (\%)} $\uparrow$ & Dist. $\downarrow$ & Task & Method & \textbf{Success (\%)} $\uparrow$ & Dist. $\downarrow$ \\
\midrule
\multirow{4}{*}{Reacher} 
& Ours & $\bm{83.5 \pm 4.06}$ & $0.377 \pm 0.119$ & \multirow{4}{*}{Rope} 
& Ours & $\bm{93.75 \pm 2.71}$ & $0.020 \pm 0.002$ \\
& Le-WM & $17.0 \pm 7.6$ & $1.305 \pm 1.109$ & 
& Le-WM & $13.0 \pm 4.3$ & $0.093 \pm 0.015$ \\
& DINO-WM & $24.5 \pm 6.2$ & $0.533 \pm 0.115$ & 
& DINO-WM & $25.0 \pm 7.91$ & $0.058 \pm 0.002$ \\
& PLDM & $40.5 \pm 9.9$ & $0.821 \pm 0.158$ & 
& PLDM & $7.0 \pm 5.34$ & $0.130 \pm 0.019$ \\
\midrule
\multirow{4}{*}{Cube} 
& Ours & $\bm{91.5 \pm 4.2}$ & $0.051 \pm 0.055$ & \multirow{4}{*}{Push-T} 
& Ours & $\bm{51.54 \pm 4.40}$ & $29.86 \pm 35.55$ \\
& Le-WM & $28.5 \pm 3.8$ & $0.133 \pm 0.023$ & 
& Le-WM & $1.5 \pm 2.00$ & $113.61 \pm 10.25$ \\
& DINO-WM & $69.0 \pm 2.95$ & $0.054 \pm 0.048$ & 
& DINO-WM & $2.52 \pm 0.04$ & $112.58 \pm 3.36$ \\
& PLDM & $28.75 \pm 6.62$ & $0.145 \pm 0.005$ & 
& PLDM & $1.50 \pm 1.37$ & $126.07 \pm 5.26$ \\
\bottomrule
\end{tabular}
}
\vspace{-10pt}
\end{table}

\vspace{-8pt}
\subsection{Nominal Planning}\label{sec:experiments_nominal}
\vspace{-4pt}
\paragraph{Reacher}
On Reacher, our nominal planner achieves the strongest performance, reaching $83.50\%$ success. Detailed numbers can be found in Table~\ref{tab:nominal_results}. The ablations in Table~\ref{tab:reacher_ablation_combined} show that this improvement is not due to a single component: using a non-Markov history-based predictor severely degrades performance, indicating that the learned Markov latent state is important for planning, while removing the multi-step dynamics loss also weakens goal reaching by reducing rollout accuracy. Replacing iLQR with CEM further lowers success, suggesting that although sampling-based planning can find reasonable trajectories, gradient-based planning better exploits the local linearity of our learned latent dynamics.

\vspace{-4pt}
\paragraph{Cube}
For OGBench Cube, we initialize each planning from an oracle-grasped state, so the experiment isolates the post-grasp transport problem of moving the cube to the target. Under this setup, our nominal planning achieves the strongest nominal planning performance, reaching $91.5\%$ success, while the best baseline, DINO-WM, only reaches $69.0\%$. Detailed numbers can be found in Table~\ref{tab:nominal_results}. This indicates that the learned Markov latent dynamics together with gradient-based planning more reliably transports the grasped object to the target than prior world-model planners.

\begin{figure*}[t] 
    \centering
    \includegraphics[width=\linewidth]{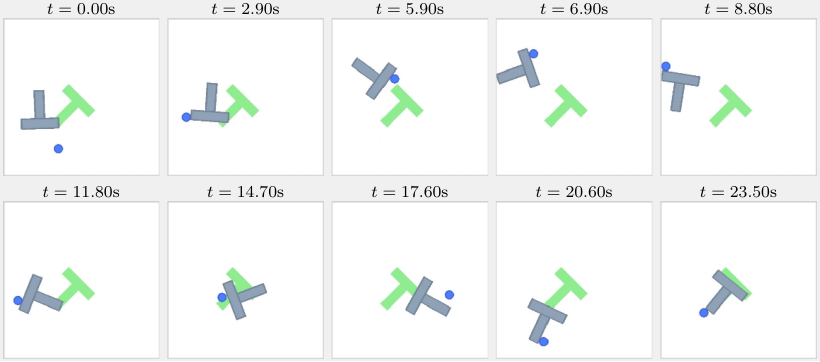}
    \vspace{-1.5em}
    \caption{\textbf{Nominal planning; Push-T.} Nominal planning with the T successfully reaching the target. Our planner executes complex maneuvers (including making and breaking contact with the T), to successfully reach the goal state. A denser set of keyframes is given in Fig. \ref{fig:pusht_nominal}.}
    \vspace{-10pt}
    \label{fig:pusht_nominal_long_rollout}
\end{figure*}

\vspace{-4pt}
\paragraph{Push-T}
Push-T is the most difficult nominal planning task in our evaluation, as it requires long-horizon contact-rich pushing with accurate control of both block position and orientation. Thus, all methods achieve lower performance compared with the other tasks. Nevertheless, our method maintains the strongest performance, achieving $51.54 \pm 4.40\%$ success and the lowest minimum block-position error of $29.86 \pm 35.55$ pixels. Figure \ref{fig:pusht_nominal_long_rollout} provides a sample rollout demonstrating the extensive maneuvers taken by our planner to successfully push the T to the goal. In contrast, DINO-WM and PLDM achieve only $2.52 \pm 0.04\%$ and $1.50 \pm 1.37\%$ success, respectively, with much larger minimum distances. These results show that, even when the task difficulty causes overall success rates to decrease, our learned Markov latent dynamics and gradient-based planner still provide the most reliable planning signal for contact-rich manipulation. 
See App.~\ref{app:extend_results} for more Push-T baselines.
\vspace{-4pt}
\paragraph{Rope}
Lastly, we evaluate on a customized two-manipulator rope task in MuJoCo, where two 7-DoF KUKA iiwa arms are mounted on opposite sides of a table and hold a rope between their end-effectors (Figure ~\ref{fig:block_diag}). Since this task is not a standard benchmark, we define a three-dimensional task state using the rope reach, height, and width, which determines the target rope geometry through symmetric target positions for the two attachment points. 
In this task, we seek to manipulate the deformable rope from the initial state to the desired rope configuration from the goal image.
Under this setting, our nominal planning achieves $93.75\%$ success, while the strongest baseline, DINO-WM, reaches only $25.0\%$. Table~\ref{tab:nominal_results} demonstrates that the learned Markov latent dynamics and gradient-based planner are especially effective in the bimanual rope setting, where small errors from either arm can significantly alter the final rope geometry.

\vspace{-4pt}
\looseness-1\paragraph{Hardware Rope} To validate our nominal planner, we replicated the rope task in the real world. Hardware implementation details are in Appendix \ref{app:hardware_implementation}. In our hardware results, our planner successfully stretches the rope from the start image (MPC step 0) to the goal image (MPC step 40) in Figure \ref{fig:rope_real_2_nominal_main}(a). In another trial (Figure \ref{fig:rope_real_nominal}, Appendix \ref{app:extend_results}), our latent space planner successfully moved the rope forward. Figure \ref{fig:rope_real_2_nominal_main}(b) and (c) plot the task space error between the configuration at each MPC step and the goal configuration and the error between the latent vector and goal latent vector, indicating a clear correlation between latent space goal error and task space error. 

\begin{figure}[H]
    \centering
    \includegraphics[width=\linewidth, trim={0.0cm 0.0cm 0.0cm 0.0cm}, clip]{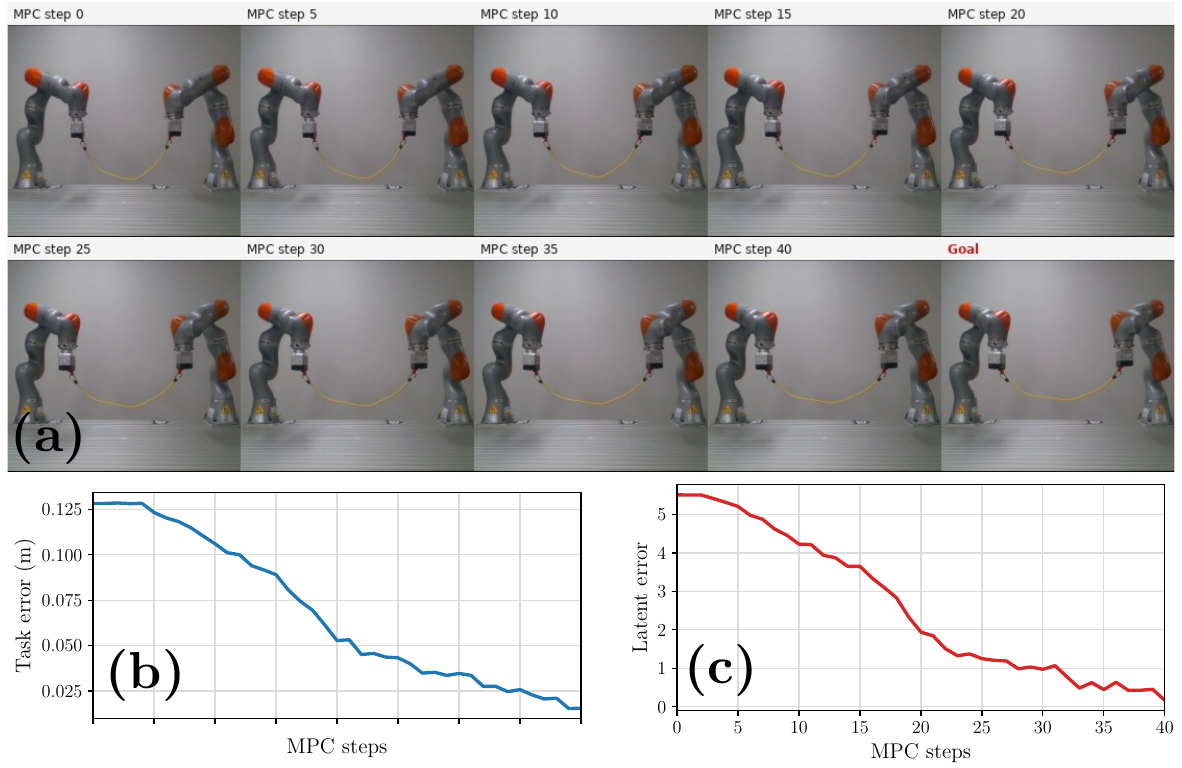}
    \vspace{-1.5em}
    \caption{\textbf{Hardware Rope.} (a) Time-lapse of iLQR-based long-horizon planning on real hardware directly from pixels without state estimation. In this example, the bimanual setup has to stretch the rope to reach the goal image. (b) Task and (c) latent space error plots demonstrating convergence to the desired goal state on hardware.}
    \label{fig:rope_real_2_nominal_main}
\end{figure}

\vspace{-7mm}
\begin{figure*}[h]
    \centering
    \includegraphics[width=\linewidth, trim={0.0cm 0.9cm 0.0cm 0.0cm}, clip]{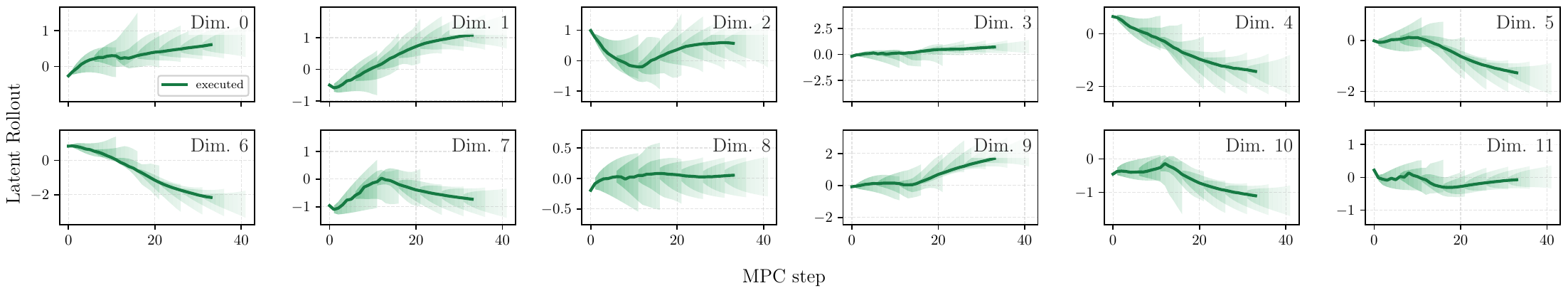}
    \vspace{-1.5em}
    \caption{\textbf{Rope.} Latent state rollout (dark green), with the projected robust tubes from SLS (shaded). }
    \label{fig:fullwidth}
\vspace{-15pt}
\end{figure*}

\subsection{Robustly-Constrained Safe Latent Space MPC}\label{sec:experiments_safe}

\begin{wrapfigure}{r}{0.55\linewidth}
    \vspace{-1.5em}
    \centering
    \includegraphics[width=\linewidth, trim={0.0cm 0.0cm 0.0cm 0.4cm}, clip]{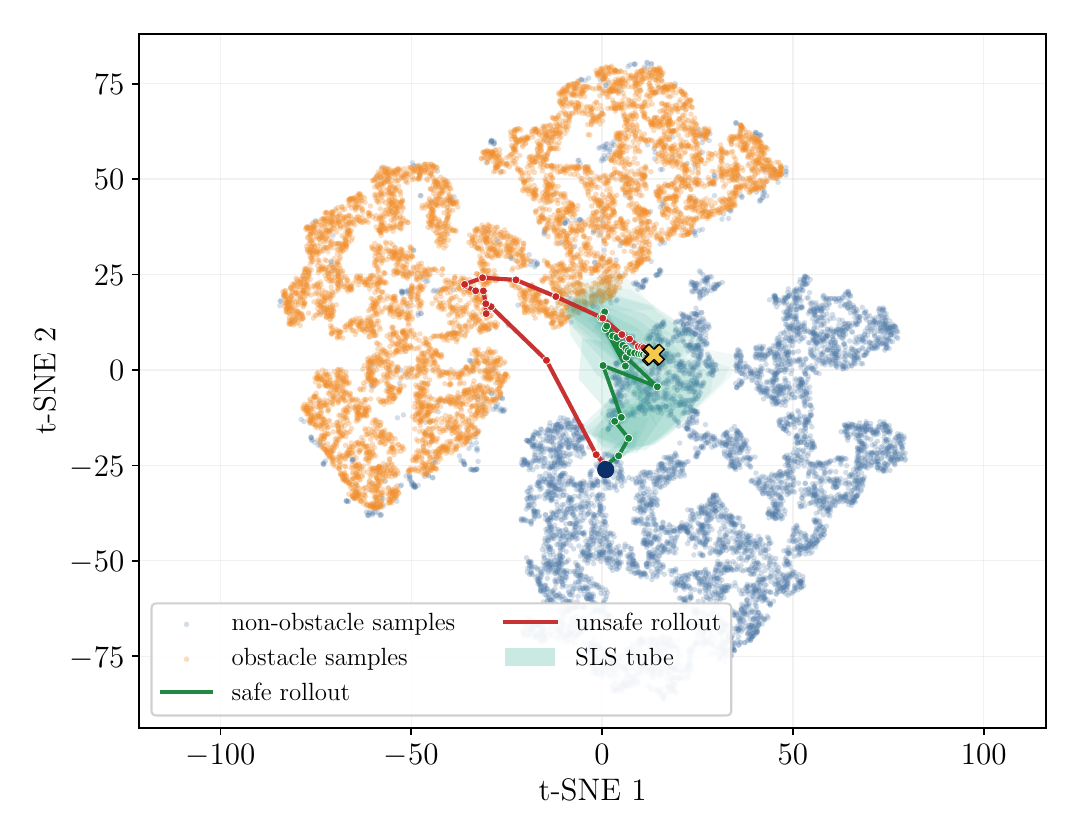}
    \vspace{-22pt}
    \caption{\textbf{Robust planning; Rope.} t-SNE visualization of the rollout. Our tubes stay predominantly in the safe region; the small overlap with the obstacle region is due to the lossy 2D t-SNE projection.}
    \label{fig:tsne_latent_rollout}
    \vspace{-1em}
\end{wrapfigure}

\vspace{-4pt}
\paragraph{Rope} For the rope task, we introduce an ellipsoidal obstacle in the center of the task space (details in Appendix \ref{app:exp_details}) that the rope must avoid while moving from the start position to a goal position. In Table \ref{tab:combined_safety_success_timing}, we find that $\text{SLS}^{2}$ is the only method that always avoids the obstacle despite perturbations due to latent model error (Figure \ref{fig:block_diag}). Conversely, the value function in the HJ-filtered baseline fails to predict safe and unsafe states, colliding with the ellipsoid. Similarly, the constrained iLQR baselines show that non-robust iLQR incurs safety violations by not accounting for modeling errors. Furthermore, SLS$^{2}$ has the highest rate for robust task completion with a rate of $88.57\%$. The rollout in the latent space (Figure \ref{fig:fullwidth}) shows that our trajectories are contained within the projected tubes. Lastly, Figure \ref{fig:tsne_latent_rollout} provides a t-SNE~\cite{van2008visualizing} visualization of the rollout against the obstacle region, demonstrating the robust but not overly conservative tubes.

\vspace{-4pt}
\paragraph{Reacher} For Reacher, we introduce joint angle constraints between -2.88 and -2.45 radians on the second joint, which must be robustly satisfied to complete the task (Figure \ref{fig:cube_and_reacher}). Both our robust method, SLS$^{2}$, and the HJ-filter baseline, using our gradient-based nominal iLQR planner as the nominal policy, robustly complete the task with a $100\%$ success rate. In contrast, the constrained iLQR variants remain safe but fail to complete the task in one instance. Notably, constrained iLQR and SLS$^{2}$ have comparable solve times, demonstrating that the robust constraint tightenings in \eqref{eq:sls_tightenings} introduce limited computational overhead.

\begin{table}[t]
\centering
\small
\caption{We evaluate SLS$^{2}$ against baseline approaches for robust task completion. For robust task completion where trajectories are safe and reach the goal, \method~performs the best compared to all baselines including the constrained iLQR (i.e., \method~without robust constraints) with or without the ID constraint \eqref{eq:sls_indist}, nominal iLQR (without any constraints), the HJ-filtered, and LPB baselines.}
\label{tab:combined_safety_success_timing}
\resizebox{\linewidth}{!}{
\begin{tabular}{cccccc}
\toprule
Task & Method & Safety (\%) $\uparrow$ & Success (\%) $\uparrow$ & \textbf{Robust Success (\%)} $\uparrow$ & Solve time/step (s) $\downarrow$\\
\midrule
\multirow{5}{*}{Reacher} 
& Ours (SLS$^2$) & 100.00 & 100.00 & \textbf{100.00} & 0.360 $\pm$ 0.002 \\
& Constrained iLQR w/ latent ID & 100.00 & 97.22 & 97.22 & 0.360 $\pm$ 0.002 \\
& Constrained iLQR w/o latent ID & 100.00 & 97.22 & 97.22 & 0.363 $\pm$ 0.003 \\
& HJ-filtered & 100.00 & 100.0 & \textbf{100.0} & -- \\
& LPB & 0.00 & 100.00 & 0.00 & -- \\
& Nominal iLQR & 0.00 & 100.0 & 0.00 & -- \\
\midrule
\multirow{5}{*}{OGBench (Cube)} 
& Ours (SLS$^2$) & 94.29 & 97.14 & \textbf{94.29} & 0.359 $\pm$ 0.016 \\
& Constrained iLQR w/ latent ID & 30.56 & 86.11 & 27.78 & 0.345 $\pm$ 0.016 \\
& Constrained iLQR w/o latent ID & 25.00 & 91.67 & 22.22 & 0.349 $\pm$ 0.013 \\
& HJ-filtered & 0.00 & 88.57 & 0.00 & -- \\
& LPB & 0.00 & 85.71 & 0.00 & -- \\
& Nominal iLQR & 0.00 & 82.86 & 0.00 & -- \\
\midrule
\multirow{6}{*}{Rope} 
& Ours (SLS$^2$) & 100.00 & 88.57 & \textbf{88.57} & 0.650 $\pm$ 0.042 \\
& Constrained iLQR w/ latent ID & 69.44 & 91.67 & 61.11 & 0.597 $\pm$ 0.026 \\
& Constrained iLQR w/o latent ID & 77.78 & 91.67 & 69.44 & 0.593 $\pm$ 0.019 \\
& HJ-filtered & 37.14 & 100.00 & 37.14 & -- \\
& LPB & 57.15 & 100.00 & 57.15 & -- \\
& Nominal iLQR & 54.29 & 100.00 & 54.29 & -- \\
\bottomrule
\end{tabular}
}
\vspace{-15pt}
\end{table}

\vspace{-10pt}
\paragraph{Cube} For the OGBench Cube task, we enforce that the gripper height must be below 9 cm after the arm has grasped the cube. Due to the model training, the arm has a natural tendency to lift the cube, thus causing the nominal-iLQR baseline to fail (Figure \ref{fig:cube_safe_unsafe_main}(b)). Across the remaining baselines, SLS$^2$ has the highest robust task completion rate for OGBench Cube by keeping the cube below the threshold (Figure \ref{fig:cube_safe_unsafe_main}(a)), indicating its ability to robustly complete tasks. Furthermore, by shifting the planning behavior from lifting the cube unnecessarily high, this experiment shows the utility of SLS$^2$ as a tool for post-hoc behavior alteration for world-model planning. 

\begin{figure}[H]
    \centering\vspace{-8pt}
    \includegraphics[width=1.0\linewidth, trim={0.0cm 0.0cm 0.0cm 0.0cm}, clip]{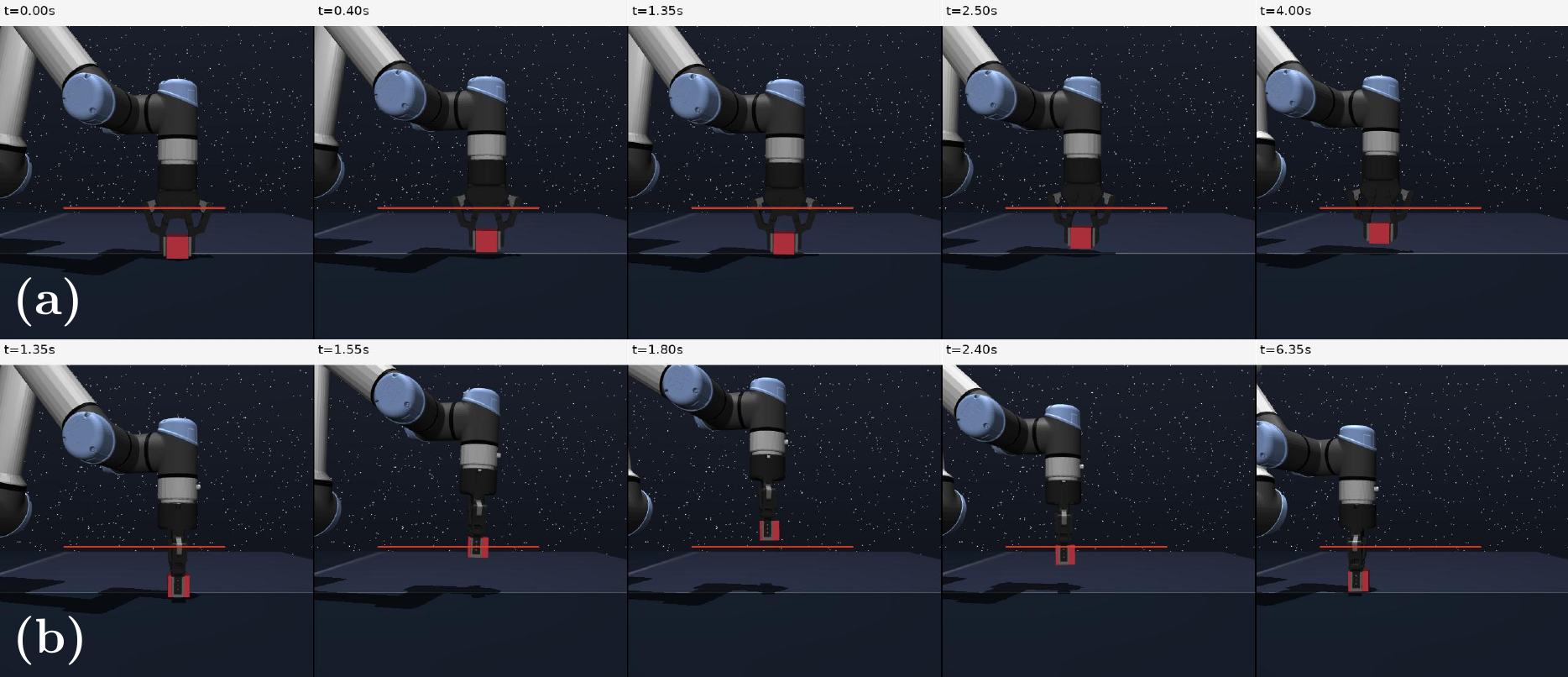}
    \vspace{-1.5em}
    \caption{\textbf{Robust planning; Cube.} Time-lapse of the rollout using (a) safe planner (\method) versus the (b) unsafe planner (nominal-iLQR). \method~maintains the cube below the height threshold (red line) while reaching the goal, unlike the nominal (unsafe) planner.}
    \label{fig:cube_safe_unsafe_main}
\end{figure}

\section{Discussion, Limitations, and Conclusion}
We propose \method, a method for probabilistically-safe latent space planning using gradient-based planners for robust visuomotor control. Methodologically, we provided an approach rigorously grounded in the uncertainty quantification and propagation tools of conformal prediction and system level synthesis. Empirically, we improve task and robust task success relative to baselines for both nominal planning and robust planning, while ensuring all components of our method, including robust planning, exploit GPU acceleration for scalability. 
\paragraph{Limitations:}  
\looseness-1Performance on environments requiring complex, sequential maneuvers, such as Push-T, remains a bottleneck because of our planner's reliance on Euclidean distance to the latent goal, which assumes an isotropic latent representation. Future work must integrate learned temporal distance metrics designed to preserve geodesic transition distances along the manifold of valid states.


\clearpage

\bibliography{example}  

\begin{thebibliography}{62}
\providecommand{\natexlab}[1]{#1}
\providecommand{\url}[1]{\texttt{#1}}
\expandafter\ifx\csname urlstyle\endcsname\relax
  \providecommand{\doi}[1]{doi: #1}\else
  \providecommand{\doi}{doi: \begingroup \urlstyle{rm}\Url}\fi

\bibitem[Ames et~al.(2019)Ames, Coogan, Egerstedt, Notomista, Sreenath, and Tabuada]{ames2019controlbarrier}
A.~D. Ames, S.~Coogan, M.~Egerstedt, G.~Notomista, K.~Sreenath, and P.~Tabuada.
\newblock Control barrier functions: Theory and applications.
\newblock In \emph{2019 18th European control conference (ECC)}, pages 3420--3431. Ieee, 2019.

\bibitem[Mitchell et~al.(2005)Mitchell, Bayen, and Tomlin]{mitchell2005timedependent}
I.~M. Mitchell, A.~M. Bayen, and C.~J. Tomlin.
\newblock A time-dependent hamilton-jacobi formulation of reachable sets for continuous dynamic games.
\newblock \emph{IEEE Transactions on automatic control}, 50\penalty0 (7):\penalty0 947--957, 2005.

\bibitem[Manchester and Slotine(2017)]{manchester2017control}
I.~R. Manchester and J.-J.~E. Slotine.
\newblock Control contraction metrics: Convex and intrinsic criteria for nonlinear feedback design.
\newblock \emph{IEEE Transactions on Automatic Control}, 62\penalty0 (6):\penalty0 3046--3053, 2017.

\bibitem[Chou et~al.(2022)Chou, Ozay, and Berenson]{chou2022safe}
G.~Chou, N.~Ozay, and D.~Berenson.
\newblock Safe output feedback motion planning from images via learned perception modules and contraction theory.
\newblock In \emph{International Workshop on the Algorithmic Foundations of Robotics}, pages 349--367. Springer, 2022.

\bibitem[Leeman et~al.(2026)Leeman, Zhan, Zeilinger, and Chou]{leeman2026vision}
A.~P. Leeman, S.~Zhan, M.~N. Zeilinger, and G.~Chou.
\newblock Vision-sls: Safe perception-based control from learned visual representations via system level synthesis.
\newblock \emph{arXiv preprint arXiv:2604.24894}, 2026.

\bibitem[Chou and Tedrake(2023)]{chou2023synthesizing}
G.~Chou and R.~Tedrake.
\newblock Synthesizing stable reduced-order visuomotor policies for nonlinear systems via sums-of-squares optimization.
\newblock In \emph{2023 62nd IEEE Conference on Decision and Control (CDC)}, pages 624--631. IEEE, 2023.

\bibitem[Dean et~al.(2021)Dean, Taylor, Cosner, Recht, and Ames]{dean2021guaranteeing}
S.~Dean, A.~Taylor, R.~Cosner, B.~Recht, and A.~Ames.
\newblock Guaranteeing safety of learned perception modules via measurement-robust control barrier functions.
\newblock In \emph{Conference on Robot Learning}, pages 654--670. PMLR, 2021.

\bibitem[Ha and Schmidhuber(2018)]{ha2018worldmodels}
D.~Ha and J.~Schmidhuber.
\newblock World models.
\newblock \emph{arXiv preprint arXiv:1803.10122}, 2\penalty0 (3):\penalty0 440, 2018.

\bibitem[Hafner et~al.(2020)Hafner, Lillicrap, Ba, and Norouzi]{hafner2020dreamer}
D.~Hafner, T.~Lillicrap, J.~Ba, and M.~Norouzi.
\newblock Dream to control: Learning behaviors by latent imagination.
\newblock In \emph{International Conference on Learning Representations}, 2020.
\newblock URL \url{https://openreview.net/forum?id=S1lOTC4tDS}.

\bibitem[Zhou et~al.(2025)Zhou, Pan, Lecun, and Pinto]{zhou2024dinowm}
G.~Zhou, H.~Pan, Y.~Lecun, and L.~Pinto.
\newblock {DINO}-{WM}: World models on pre-trained visual features enable zero-shot planning.
\newblock In \emph{Proceedings of the 42nd International Conference on Machine Learning}, volume 267 of \emph{Proceedings of Machine Learning Research}, pages 79115--79135. PMLR, 13--19 Jul 2025.

\bibitem[Hansen et~al.(2024)Hansen, Su, and Wang]{hansen2024tdmpc2}
N.~Hansen, H.~Su, and X.~Wang.
\newblock Td-mpc2: Scalable, robust world models for continuous control.
\newblock In \emph{International Conference on Learning Representations}, volume 2024, pages 47376--47405, 2024.

\bibitem[Castaneda et~al.(2023)Castaneda, Nishimura, McAllister, Sreenath, and Gaidon]{castaneda2023distribution}
F.~Castaneda, H.~Nishimura, R.~T. McAllister, K.~Sreenath, and A.~Gaidon.
\newblock In-distribution barrier functions: Self-supervised policy filters that avoid out-of-distribution states.
\newblock In \emph{Learning for Dynamics and Control Conference}, pages 286--299. PMLR, 2023.

\bibitem[Nakamura et~al.(2025)Nakamura, Peters, and Bajcsy]{nakamura2025generalizing}
K.~Nakamura, L.~Peters, and A.~Bajcsy.
\newblock Generalizing safety beyond collision-avoidance via latent-space reachability analysis.
\newblock \emph{arXiv preprint arXiv:2502.00935}, 2025.

\bibitem[Anderson et~al.(2019)Anderson, Doyle, Low, and Matni]{anderson2019system}
J.~Anderson, J.~C. Doyle, S.~H. Low, and N.~Matni.
\newblock System level synthesis.
\newblock \emph{Annual Reviews in Control}, 47:\penalty0 364--393, 2019.

\bibitem[Hafner et~al.(2019)Hafner, Lillicrap, Fischer, Villegas, Ha, Lee, and Davidson]{hafner2019planet}
D.~Hafner, T.~Lillicrap, I.~Fischer, R.~Villegas, D.~Ha, H.~Lee, and J.~Davidson.
\newblock Learning latent dynamics for planning from pixels.
\newblock In \emph{International conference on machine learning}, pages 2555--2565. PMLR, 2019.

\bibitem[Hafner et~al.(2021)Hafner, Lillicrap, Norouzi, and Ba]{hafner2021dreamerv2}
D.~Hafner, T.~P. Lillicrap, M.~Norouzi, and J.~Ba.
\newblock Mastering atari with discrete world models.
\newblock In \emph{International Conference on Learning Representations}, 2021.
\newblock URL \url{https://openreview.net/forum?id=0oabwyZbOu}.

\bibitem[Hafner et~al.(2023)Hafner, Pasukonis, Ba, and Lillicrap]{hafner2023dreamerv3}
D.~Hafner, J.~Pasukonis, J.~Ba, and T.~Lillicrap.
\newblock Mastering diverse domains through world models.
\newblock \emph{arXiv preprint arXiv:2301.04104}, 2023.

\bibitem[Hafner et~al.(2025)Hafner, Yan, and Lillicrap]{hafner2025dreamerv4}
D.~Hafner, W.~Yan, and T.~Lillicrap.
\newblock Training agents inside of scalable world models.
\newblock \emph{arXiv preprint arXiv:2509.24527}, 2025.

\bibitem[Hansen et~al.(2022)Hansen, Su, and Wang]{hansen2022tdmpc}
N.~A. Hansen, H.~Su, and X.~Wang.
\newblock Temporal difference learning for model predictive control.
\newblock In \emph{Proceedings of the 39th International Conference on Machine Learning}, volume 162 of \emph{Proceedings of Machine Learning Research}, pages 8387--8406. PMLR, 17--23 Jul 2022.

\bibitem[LeCun et~al.(2022)]{lecun2022path}
Y.~LeCun et~al.
\newblock A path towards autonomous machine intelligence version 0.9. 2, 2022-06-27.
\newblock \emph{Open Review}, 62\penalty0 (1):\penalty0 1--62, 2022.

\bibitem[Assran et~al.(2023)Assran, Duval, Misra, Bojanowski, Vincent, Rabbat, LeCun, and Ballas]{assran2023ijepa}
M.~Assran, Q.~Duval, I.~Misra, P.~Bojanowski, P.~Vincent, M.~Rabbat, Y.~LeCun, and N.~Ballas.
\newblock Self-supervised learning from images with a joint-embedding predictive architecture.
\newblock In \emph{Proceedings of the IEEE/CVF conference on computer vision and pattern recognition}, pages 15619--15629, 2023.

\bibitem[Bardes et~al.(2024)Bardes, Garrido, Ponce, Chen, Rabbat, LeCun, Assran, and Ballas]{bardes2024vjepa}
A.~Bardes, Q.~Garrido, J.~Ponce, X.~Chen, M.~Rabbat, Y.~LeCun, M.~Assran, and N.~Ballas.
\newblock Revisiting feature prediction for learning visual representations from video.
\newblock \emph{Transactions on Machine Learning Research}, 2024.
\newblock ISSN 2835-8856.
\newblock URL \url{https://openreview.net/forum?id=QaCCuDfBk2}.
\newblock Featured Certification.

\bibitem[Assran et~al.(2025)Assran, Bardes, Fan, Garrido, Howes, Muckley, Rizvi, Roberts, Sinha, Zholus, et~al.]{assran2025vjepa2}
M.~Assran, A.~Bardes, D.~Fan, Q.~Garrido, R.~Howes, M.~Muckley, A.~Rizvi, C.~Roberts, K.~Sinha, A.~Zholus, et~al.
\newblock V-jepa 2: Self-supervised video models enable understanding, prediction and planning.
\newblock \emph{arXiv preprint arXiv:2506.09985}, 2025.

\bibitem[Sobal et~al.(2026)Sobal, Zhang, Cho, Balestriero, Rudner, and LeCun]{sobal2025pldm}
U.~Sobal, W.~Zhang, K.~Cho, R.~Balestriero, T.~G. Rudner, and Y.~LeCun.
\newblock Learning from reward-free offline data: A case for planning with latent dynamics models.
\newblock \emph{Advances in Neural Information Processing Systems}, 38:\penalty0 43905--43941, 2026.

\bibitem[Oquab et~al.(2023)Oquab, Darcet, Moutakanni, Vo, Szafraniec, Khalidov, Fernandez, Haziza, Massa, El-Nouby, et~al.]{oquab2024dinov2}
M.~Oquab, T.~Darcet, T.~Moutakanni, H.~Vo, M.~Szafraniec, V.~Khalidov, P.~Fernandez, D.~Haziza, F.~Massa, A.~El-Nouby, et~al.
\newblock Dinov2: Learning robust visual features without supervision.
\newblock \emph{arXiv preprint arXiv:2304.07193}, 2023.

\bibitem[Micheli et~al.(2023)Micheli, Alonso, and Fleuret]{micheli2023iris}
V.~Micheli, E.~Alonso, and F.~Fleuret.
\newblock Transformers are sample-efficient world models.
\newblock In \emph{The Eleventh International Conference on Learning Representations}, 2023.
\newblock URL \url{https://openreview.net/forum?id=vhFu1Acb0xb}.

\bibitem[Alonso et~al.(2024)Alonso, Jelley, Micheli, Kanervisto, Storkey, Pearce, and Fleuret]{alonso2024diamond}
E.~Alonso, A.~Jelley, V.~Micheli, A.~Kanervisto, A.~Storkey, T.~Pearce, and F.~Fleuret.
\newblock Diffusion for world modeling: Visual details matter in atari.
\newblock \emph{Advances in Neural Information Processing Systems}, 37:\penalty0 58757--58791, 2024.

\bibitem[Micheli et~al.(2024)Micheli, Alonso, and Fleuret]{micheli2024deltairis}
V.~Micheli, E.~Alonso, and F.~Fleuret.
\newblock Efficient world models with context-aware tokenization.
\newblock In \emph{Forty-first International Conference on Machine Learning}, 2024.
\newblock URL \url{https://openreview.net/forum?id=BiWIERWBFX}.

\bibitem[Bruce et~al.(2024)Bruce, Dennis, Edwards, Parker-Holder, Shi, Hughes, Lai, Mavalankar, Steigerwald, Apps, et~al.]{bruce2024genie}
J.~Bruce, M.~D. Dennis, A.~Edwards, J.~Parker-Holder, Y.~Shi, E.~Hughes, M.~Lai, A.~Mavalankar, R.~Steigerwald, C.~Apps, et~al.
\newblock Genie: Generative interactive environments.
\newblock In \emph{Forty-first International Conference on Machine Learning}, 2024.

\bibitem[Maes et~al.(2026)Maes, Lidec, Scieur, LeCun, and Balestriero]{lewm}
L.~Maes, Q.~L. Lidec, D.~Scieur, Y.~LeCun, and R.~Balestriero.
\newblock Leworldmodel: Stable end-to-end joint-embedding predictive architecture from pixels.
\newblock \emph{arXiv preprint arXiv:2603.19312}, 2026.

\bibitem[Ames et~al.(2016)Ames, Xu, Grizzle, and Tabuada]{ames2017controlbarrier}
A.~D. Ames, X.~Xu, J.~W. Grizzle, and P.~Tabuada.
\newblock Control barrier function based quadratic programs for safety critical systems.
\newblock \emph{IEEE Transactions on Automatic Control}, 62\penalty0 (8):\penalty0 3861--3876, 2016.

\bibitem[Fisac et~al.(2018)Fisac, Akametalu, Zeilinger, Kaynama, Gillula, and Tomlin]{fisac2019general}
J.~F. Fisac, A.~K. Akametalu, M.~N. Zeilinger, S.~Kaynama, J.~Gillula, and C.~J. Tomlin.
\newblock A general safety framework for learning-based control in uncertain robotic systems.
\newblock \emph{IEEE Transactions on Automatic Control}, 64\penalty0 (7):\penalty0 2737--2752, 2018.

\bibitem[Wabersich and Zeilinger(2021)]{wabersich2021predictive}
K.~P. Wabersich and M.~N. Zeilinger.
\newblock A predictive safety filter for learning-based control of constrained nonlinear dynamical systems.
\newblock \emph{Automatica}, 129:\penalty0 109597, 2021.

\bibitem[Wabersich et~al.(2023)Wabersich, Taylor, Choi, Sreenath, Tomlin, Ames, and Zeilinger]{wabersich2023datadrivensafetyfilters}
K.~P. Wabersich, A.~J. Taylor, J.~J. Choi, K.~Sreenath, C.~J. Tomlin, A.~D. Ames, and M.~N. Zeilinger.
\newblock Data-driven safety filters: Hamilton-jacobi reachability, control barrier functions, and predictive methods for uncertain systems.
\newblock \emph{IEEE Control Systems Magazine}, 43\penalty0 (5):\penalty0 137--177, 2023.

\bibitem[Robey et~al.(2020)Robey, Hu, Lindemann, Zhang, Dimarogonas, Tu, and Matni]{robey2020learningcbfs}
A.~Robey, H.~Hu, L.~Lindemann, H.~Zhang, D.~V. Dimarogonas, S.~Tu, and N.~Matni.
\newblock Learning control barrier functions from expert demonstrations.
\newblock In \emph{2020 59th IEEE Conference on Decision and Control (CDC)}, pages 3717--3724. Ieee, 2020.

\bibitem[Lindemann et~al.(2024)Lindemann, Robey, Jiang, Das, Tu, and Matni]{lindemann2024robustoutputcbfs}
L.~Lindemann, A.~Robey, L.~Jiang, S.~Das, S.~Tu, and N.~Matni.
\newblock Learning robust output control barrier functions from safe expert demonstrations.
\newblock \emph{IEEE Open Journal of Control Systems}, 3:\penalty0 158--172, 2024.

\bibitem[Kang et~al.(2022)Kang, Gradu, Choi, Janner, Tomlin, and Levine]{kang2022lyapunovdensity}
K.~Kang, P.~Gradu, J.~J. Choi, M.~Janner, C.~Tomlin, and S.~Levine.
\newblock Lyapunov density models: Constraining distribution shift in learning-based control.
\newblock In \emph{International Conference on Machine Learning}, pages 10708--10733. PMLR, 2022.

\bibitem[Faroni et~al.(2025)Faroni, Odesco, Zanchettin, and Rocco]{faroni2025uncertainty}
M.~Faroni, C.~Odesco, A.~M. Zanchettin, and P.~Rocco.
\newblock Uncertainty-aware planning with inaccurate models for robotized liquid handling.
\newblock In \emph{2025 IEEE/RSJ International Conference on Intelligent Robots and Systems (IROS)}, pages 17725--17731. IEEE, 2025.

\bibitem[Cao et~al.(2025)Cao, Bloch, and Coogan]{cao2025optimistic}
M.~E. Cao, M.~Bloch, and S.~Coogan.
\newblock An optimistic approach to cost-aware predictive control.
\newblock \emph{Automatica}, 176:\penalty0 112263, 2025.

\bibitem[Chou et~al.(2021)Chou, Ozay, and Berenson]{chou2021model}
G.~Chou, N.~Ozay, and D.~Berenson.
\newblock Model error propagation via learned contraction metrics for safe feedback motion planning of unknown systems.
\newblock In \emph{2021 60th IEEE Conference on Decision and Control (CDC)}, pages 3576--3583. IEEE, 2021.

\bibitem[Knuth et~al.(2021)Knuth, Chou, Ozay, and Berenson]{knuth2021planning}
C.~Knuth, G.~Chou, N.~Ozay, and D.~Berenson.
\newblock Planning with learned dynamics: Probabilistic guarantees on safety and reachability via lipschitz constants.
\newblock \emph{IEEE Robotics and Automation Letters}, 6\penalty0 (3):\penalty0 5129--5136, 2021.

\bibitem[Knuth et~al.(2023)Knuth, Chou, Reese, and Moore]{knuth2023statistical}
C.~Knuth, G.~Chou, J.~Reese, and J.~Moore.
\newblock Statistical safety and robustness guarantees for feedback motion planning of unknown underactuated stochastic systems.
\newblock In \emph{2023 IEEE International Conference on Robotics and Automation (ICRA)}, pages 12700--12706. IEEE, 2023.

\bibitem[Suh et~al.(2023)Suh, Chou, Dai, Yang, Gupta, and Tedrake]{suh2023fighting}
H.~T. Suh, G.~Chou, H.~Dai, L.~Yang, A.~Gupta, and R.~Tedrake.
\newblock Fighting uncertainty with gradients: Offline reinforcement learning via diffusion score matching.
\newblock In \emph{Conference on Robot Learning}, pages 2878--2904. PMLR, 2023.

\bibitem[Dodeja et~al.(2025)Dodeja, Schmeckpeper, Vats, Weng, Jia, Konidaris, and Tellex]{dodeja2025accelerating}
L.~Dodeja, K.~Schmeckpeper, S.~Vats, T.~Weng, M.~Jia, G.~Konidaris, and S.~Tellex.
\newblock Accelerating residual reinforcement learning with uncertainty estimation.
\newblock \emph{IEEE Robotics and Automation Letters}, 11\penalty0 (1):\penalty0 970--977, 2025.

\bibitem[Chua et~al.(2018)Chua, Calandra, McAllister, and Levine]{chua2018deep}
K.~Chua, R.~Calandra, R.~McAllister, and S.~Levine.
\newblock Deep reinforcement learning in a handful of trials using probabilistic dynamics models.
\newblock \emph{Advances in neural information processing systems}, 31, 2018.

\bibitem[Seo et~al.(2025)Seo, Nakamura, and Bajcsy]{seo2025unisafe}
J.~Seo, K.~Nakamura, and A.~Bajcsy.
\newblock Uncertainty-aware latent safety filters for avoiding out-of-distribution failures.
\newblock \emph{arXiv preprint arXiv:2505.00779}, 2025.

\bibitem[Nakamura et~al.(2025)Nakamura, Bishop, Man, Johnson, Manchester, and Bajcsy]{nakamura2025latentcbf}
K.~Nakamura, A.~L. Bishop, S.~Man, A.~M. Johnson, Z.~Manchester, and A.~Bajcsy.
\newblock How to train your latent control barrier function: Smooth safety filtering under hard-to-model constraints.
\newblock \emph{arXiv preprint arXiv:2511.18606}, 2025.

\bibitem[Anand and Kolathaya(2025)]{anand2025safetycertification}
M.~Anand and S.~Kolathaya.
\newblock Safety certification in the latent space using control barrier functions and world models.
\newblock \emph{arXiv preprint arXiv:2507.13871}, 2025.

\bibitem[Sun and Song(2026)]{sun2025latentpolicybarrierlearning}
Z.~Sun and S.~Song.
\newblock Latent policy barrier: Learning robust visuomotor policies by staying in-distribution.
\newblock \emph{Advances in Neural Information Processing Systems}, 38:\penalty0 174280--174305, 2026.

\bibitem[Mao et~al.(2024)Mao, Sobolewski, and Ruchkin]{mao2024safe}
Z.~Mao, C.~Sobolewski, and I.~Ruchkin.
\newblock How safe am i given what i see? calibrated prediction of safety chances for image-controlled autonomy.
\newblock In \emph{6th Annual Learning for Dynamics \& Control Conference}, pages 1370--1387. PMLR, 2024.

\bibitem[Balestriero and LeCun(2025)]{balestriero2025lejepa}
R.~Balestriero and Y.~LeCun.
\newblock Lejepa: Provable and scalable self-supervised learning without the heuristics.
\newblock \emph{arXiv preprint arXiv:2511.08544}, 2025.

\bibitem[Wang et~al.(2026)Wang, Bounou, Zhou, Balestriero, Rudner, LeCun, and Ren]{wang2026straightening}
Y.~Wang, O.~Bounou, G.~Zhou, R.~Balestriero, T.~G. Rudner, Y.~LeCun, and M.~Ren.
\newblock Temporal straightening for latent planning.
\newblock \emph{arXiv preprint arXiv:2603.12231}, 2026.

\bibitem[Tibshirani et~al.(2019)Tibshirani, Foygel~Barber, Candes, and Ramdas]{tibshirani2019conformal}
R.~J. Tibshirani, R.~Foygel~Barber, E.~Candes, and A.~Ramdas.
\newblock Conformal prediction under covariate shift.
\newblock \emph{Advances in neural information processing systems}, 32, 2019.

\bibitem[Lei and Wasserman(2014)]{lei2014distribution}
J.~Lei and L.~Wasserman.
\newblock Distribution-free prediction bands for non-parametric regression.
\newblock \emph{Journal of the Royal Statistical Society Series B: Statistical Methodology}, 76\penalty0 (1):\penalty0 71--96, 2014.

\bibitem[Fang and Chou(2026)]{fang2026safe}
J.~Fang and G.~Chou.
\newblock Safe large-scale robust nonlinear mpc in milliseconds via reachability-constrained system level synthesis on the gpu.
\newblock \emph{arXiv preprint arXiv:2604.07644}, 2026.

\bibitem[Srinivasan et~al.(2026)Srinivasan, Leeman, and Chou]{srinivasan2026safety}
A.~Srinivasan, A.~Leeman, and G.~Chou.
\newblock Safety beyond the training data: Robust out-of-distribution mpc via conformalized system level synthesis.
\newblock In \emph{8th Annual Learning for Dynamics and Control Conference}, 2026.

\bibitem[Tassa et~al.(2018)Tassa, Doron, Muldal, Erez, Li, Casas, Budden, Abdolmaleki, Merel, Lefrancq, et~al.]{tassa2018deepmind}
Y.~Tassa, Y.~Doron, A.~Muldal, T.~Erez, Y.~Li, D.~d.~L. Casas, D.~Budden, A.~Abdolmaleki, J.~Merel, A.~Lefrancq, et~al.
\newblock Deepmind control suite.
\newblock \emph{arXiv preprint arXiv:1801.00690}, 2018.

\bibitem[Park et~al.(2025)Park, Frans, Eysenbach, and Levine]{park2025ogbench}
S.~Park, K.~Frans, B.~Eysenbach, and S.~Levine.
\newblock Ogbench: Benchmarking offline goal-conditioned rl.
\newblock In \emph{International Conference on Learning Representations}, volume 2025, pages 94937--94982, 2025.

\bibitem[Li and Todorov(2004)]{li2004iterative}
W.~Li and E.~Todorov.
\newblock Iterative linear quadratic regulator design for nonlinear biological movement systems.
\newblock In \emph{First International Conference on Informatics in Control, Automation and Robotics}, volume~2, pages 222--229. SciTePress, 2004.

\bibitem[Van~der Maaten and Hinton(2008)]{van2008visualizing}
L.~Van~der Maaten and G.~Hinton.
\newblock Visualizing data using t-sne.
\newblock \emph{Journal of machine learning research}, 2008.

\bibitem[Williams et~al.(2017)Williams, Aldrich, and Theodorou]{williams2017model}
G.~Williams, A.~Aldrich, and E.~A. Theodorou.
\newblock Model predictive path integral control: From theory to parallel computation.
\newblock \emph{Journal of Guidance, Control, and Dynamics}, 40\penalty0 (2):\penalty0 344--357, 2017.

\bibitem[Tedrake and the Drake Development~Team(2019)]{drake}
R.~Tedrake and the Drake Development~Team.
\newblock Drake: Model-based design and verification for robotics, 2019.
\newblock URL \url{https://drake.mit.edu}.

\end{thebibliography}

\newpage
\appendix

\setcounter{theorem}{0}
\renewcommand{\thetheorem}{\Alph{section}.\arabic{theorem}}

\begin{center}
    \Large \textbf{Appendices}
\end{center}

In the following, we provide an overview of our appendices. In App. \ref{app:training_details}, we provide training details for the latent world models, including the offline data collection procedure, shared HDF5 dataset format, optimization settings, and task-specific latent dynamics network architectures. In App. \ref{app:sls}, we review the system level synthesis (SLS) formulation used to compute robust feedback policies and reachable tubes for the learned latent dynamics, and describe how nominal constraints are tightened to account for calibrated model error. In App. \ref{app:proofs}, we prove the probabilistic containment guarantee for the conformal SLS tubes. In App. \ref{app:exp_details}, we provide experimental implementation details, including the benchmark setup, nominal planning objectives, world-model baselines, implementation of Problem \eqref{eq:cp_sls_mpc}, calibrated model-error bounds, latent constraint classifiers, in-distribution latent support constraints, safety baselines, and the real-world rope manipulation setup. In App. \ref{app:extend_results}, we provide additional experimental results, including nominal trajectory time-lapses, safe versus unsafe trajectory comparisons, latent reachable-tube visualizations, and expanded quantitative tables for nominal planning and robust task-completion baselines.

\section{Training Details}\label{app:training_details}

\paragraph{Offline Dataset}

We train the latent world model from offline trajectories collected independently for each task. Each trajectory consists of a sequence of image observations and actions,
\begin{equation}
    \tau = \{(o_t, a_t)\}_{t=0}^{T-1},
\end{equation}
where $o_t$ is a rendered RGB observation and $a_t$ is the action applied between consecutive observations. The datasets are collected before world model training and no online interaction is used during model learning. The world model is trained only from the observation-action sequences, while additional simulator states, rewards, and task-specific quantities are stored for analysis, visualization, and evaluation.

Across environments, we use task-specific data-generation policies to produce trajectories with sufficient interaction coverage. For expert-control tasks, such as Reacher, trajectories are collected by rolling out a trained expert policy. For manipulation tasks that require structured behavior, such as OGBench Cube, trajectories are generated using an oracle planner that constructs task-space keyframes and executes the corresponding control sequence. For the two-manipulator rope environment, we use randomized task-space trajectory generators, including spline- and waypoint-based policies, to cover a diverse range of rope configurations and manipulator motions. This produces datasets that are not tied to a single evaluation goal, but instead expose the world model to local dynamics over a broad set of feasible interactions.

All datasets are stored in a common HDF5 format. For each episode, we store the rendered pixel observations, applied actions, episode length, episode offset, episode seed, and per-step episode/step indices. Actions are padded with a final invalid entry so that the action array has the same temporal length as the observation array; this makes it straightforward to sample fixed-length training windows while ignoring the last action of each episode. Episodes shorter than a minimum length are discarded. The collection loop either runs for a fixed number of episodes or until a specified target number of transitions is reached. This common storage format allows the same world-model training pipeline to be reused across all environments.

\paragraph{Training the World Model}

Across experiments, we train the model using AdamW with learning rate $5\times10^{-5}$, weight decay $10^{-3}$, mixed-precision training, and gradient clipping with maximum norm $1.0$. Model checkpoints are saved after each epoch, and the same general training pipeline is used across all tasks, with only task-dependent choices such as latent dimension, Markov order, rollout horizon, and loss weights adjusted for each environment.

\paragraph{MLP Latent Dynamics Models}
All four tasks use the same vision backbone: a tiny ViT encoder on $224\times224$ RGB images with
patch size $14$.  The encoder output is projected to a task-specific latent $z_t\in\mathbb{R}^{n_z}$.
The dynamics model then forms a Markov latent state
\[
s_t=[z_t,\Delta z_t,\ldots,\Delta^K z_t]\in\mathbb{R}^{(K+1)n_z},
\]
where $K$ is the desired Markov order.  A feed-forward MLP predicts the next Markov
state from the current Markov state and action,
\[
\hat s_{t+1}=f_{\phi}(s_t,u_t).
\]
Here the listed action dimension is the number of control channels concatenated into the predictor
input; since all runs use frameskip $1$, this is also the effective action dimension.  Reacher and
Push-T use two-dimensional actions, Rope uses a three-dimensional endpoint action, and OGBench Cube
uses the five-dimensional environment action vector.  The predictor is trained by autoregressively
rolling this one-step MLP for \texttt{num\_preds}=5 future steps.

\begin{table}[h]
\centering
\small
\caption{Model architecture and state dimensions across tasks.}
\begin{tabular}{lcccccc}
\toprule
Task & $n_z$ & $K$ & $\dim(s_t)$ & $\dim(v_t)$ & MLP Size & Straightening \\
\midrule
OGBench-Cube & $12$ & $1$ & $24$ & $5$ & $512{\times}2$ & yes \\
Reacher & $5$ & $1$ & $10$ & $2$ & $512{\times}2$ & no \\
Push-T & $48$ & $2$ & $144$ & $2$ & $512{\times}3$ & no \\
Rope & $12$ & $1$ & $24$ & $3$ & $512{\times}2$ & yes \\
\bottomrule
\end{tabular}
\label{tab:model_architecture_dims}
\end{table}

The MLP column reports hidden width times number of hidden layers, with GELU nonlinearities and a
linear output layer.  The MLP size column only considers the latent dynamics predictor, excluding the
ViT encoder and latent projector.  The SIGReg term uses the same sketch size in all runs: $17$
quadrature knots and $1024$ random projections, weighted by
$\lambda_{\mathrm{SIG}}=0.005$.  Temporal straightening is active only for the OGBench Cube and Rope
runs.

\section{System Level Synthesis}\label{app:sls}

\paragraph{System Level Synthesis.}
We use state-feedback SLS to compute robust feedback policies and reachable tubes for the learned latent dynamics. Let $s_k\in\mathbb{R}^{n_s}$ be the Markov latent state and $u_k\in\mathbb{R}^{n_u}$ the control. Around a nominal trajectory $(\zeta_{0:T},v_{0:T-1})$ satisfying $\zeta_{k+1}=f_\phi(\zeta_k,v_k)$, we linearize
\[
    \Delta s_{k+1}=A_k\Delta s_k+B_k\Delta u_k+E_k w_k,\qquad
    w_k\in\mathcal{B}^{n_s},
\]
where $A_k=\nabla_s f_\phi(\zeta_k,v_k)$, $B_k=\nabla_u f_\phi(\zeta_k,v_k)$, and $E_k\mathcal{B}^{n_s}$ is a calibrated latent prediction-error set.

SLS parameterizes the closed-loop system by response matrices $\Phi^s_{k,j}$ and $\Phi^u_{k,j}$ such that
\[
    s_k-\zeta_k=\sum_{j=0}^{k-1}\Phi^s_{k,j}w_j,\qquad
    u_k-v_k=\sum_{j=0}^{k-1}\Phi^u_{k,j}w_j .
\]
These responses are valid if they satisfy the finite-horizon SLS constraints
\[
    \Phi^s_{k+1,j}=A_k\Phi^s_{k,j}+B_k\Phi^u_{k,j},\quad j<k,
    \qquad
    \Phi^s_{k+1,k}=E_k .
\]
Then defining $\oplus$ and $\bigoplus$ as the Minkowski sum of two sets (i.e., $\mathcal{A} \oplus \mathcal{B} = \{a + b\mid a \in\mathcal{A}, b\in\mathcal{B}\}$ for sets, $\mathcal{A}$ and $\mathcal{B}$), the corresponding reachable tubes are
\[
    \mathcal{R}_k^s=\zeta_k\oplus\bigoplus_{j=0}^{k-1}\Phi^s_{k,j}\mathcal{B}^{n_s},
    \qquad
    \mathcal{R}_k^u=v_k\oplus\bigoplus_{j=0}^{k-1}\Phi^u_{k,j}\mathcal{B}^{n_s}.
\]
For constraints $\mathcal{F}:=\{(s,u)\mid g_i(s,u)+b_i\le0,\ i=1,\ldots,n_c\}$, we impose robust satisfaction by tightening the nominal constraints. Linearizing $g_i$ around $(\zeta_k,v_k)$ gives the sufficient condition
\[
    g_i(\zeta_k,v_k)+b_i+
    \sum_{j=0}^{k-1}
    \left\|
        \nabla_s g_i(\zeta_k,v_k)\Phi^s_{k,j}
        +
        \nabla_u g_i(\zeta_k,v_k)\Phi^u_{k,j}
    \right\|_2
    \le0 ,
\]
for all $i$ and $k$. The SLS MPC problem jointly optimizes $(\zeta,v,\Phi^s,\Phi^u)$ subject to nominal dynamics, the SLS response constraints, and these tightened constraints.

\section{Proofs}\label{app:proofs}

\begin{theorem}[Probabilistic containment of conformal SLS tubes]
\label{thm:cp_sls_latent_tube}
Consider the latent closed-loop dynamics
$s_{k+1}=f_\phi(s_k,u_k)+e_k$,
for $k=0,\ldots,T-1$.
Let $q$ be the split conformal quantile of the calibration scores
$r_i
    =
    e_i^\top \Sigma_\psi(s_i,u_i)^{-1}e_i$ and
    $e_i:=s_i^+-f_\phi(s_i,u_i)$, 
computed at per-step miscoverage $\bar\delta=\delta/T$. For each closed-loop
transition, define the test score
$r_k^{\mathrm{test}} := e_k^\top \Sigma_\psi(s_k,u_k)^{-1}e_k$. 
Assume that, for each $k$, $r_k^{\mathrm{test}}$ is exchangeable with the calibration
scores. Let $\Sigma_\psi(s,u)=L_\psi(s,u)L_\psi(s,u)^\top$ and define the local
conformal error set
$\mathcal{E}(s,u)
    :=
    \sqrt{q}\,L_\psi(s,u)\mathcal{B}^{n_s}$ as defined in \eqref{eq:cp_set_prediction}. 
Suppose the SLS controller is synthesized with disturbance matrices $E_k$ such that
the realized conformal error sets are contained in the SLS disturbance sets,
$\mathcal{E}(s_k,u_k)\subseteq E_k\mathcal{B}^{n_s}$, for $k=0,\ldots,T-1$. 
If the SLS controller is robustly correct for all disturbances
$w_k\in\mathcal{B}^{n_s}$ in
$\Delta s_{k+1}
    =
    A_k\Delta s_k+B_k\Delta u_k+E_kw_k$, where $\Delta s_k := s_k - \zeta_k$ and $\Delta u_k := u_k - v_k$, 
then
$\mathbb{P}\left[
        (s_k\in\mathcal{R}_k^s,\;
        \forall k=0,\ldots,T) \wedge\;
        (u_k\in\mathcal{R}_k^u,\;
        \forall k=0,\ldots,T-1)
    \right]
    \ge 1-\delta$.
\end{theorem}

\begin{proof}
By split CP, exchangeability gives
\[
    \mathbb{P}\left[r_k^{\mathrm{test}}\le q\right]\ge 1-\bar\delta
\]
for each $k$. Since
\[
    r_k^{\mathrm{test}}\le q
    \quad\Longleftrightarrow\quad
    e_k\in
    \sqrt{q}\,L_\psi(s_k,u_k)\mathcal{B}^{n_s}
    =
    \mathcal{E}(s_k,u_k),
\]
a union bound over the $T$ transitions yields
\[
    \mathbb{P}\left[
        e_k\in\mathcal{E}(s_k,u_k),
        \ \forall k=0,\ldots,T-1
    \right]
    \ge
    1-T\bar\delta
    =
    1-\delta .
\]
On this event, the assumed containment
$\mathcal{E}(s_k,u_k)\subseteq E_k\mathcal{B}^{n_s}$ implies that each realized
prediction error can be written as
\[
    e_k=E_kw_k,
    \qquad
    w_k\in\mathcal{B}^{n_s}.
\]
Thus the realized closed-loop trajectory is generated by a disturbance sequence inside
the uncertainty set used by SLS. Robust correctness of SLS then implies
$s_k\in\mathcal{R}_k^s$ for $k=0,\ldots,T$ and
$u_k\in\mathcal{R}_k^u$ for $k=0,\ldots,T-1$. Therefore the tube-containment event
holds with probability at least $1-\delta$.
\end{proof}

\section{Experimental Implementation}\label{app:exp_details}
In the following sections, we discuss the experimental implementations for two of our central contributions: our \textbf{\textit{nominal planner}} and our \textbf{\textit{robust planner}}. Section \ref{app:exp_setup} discusses the experimental setup for our nominal experiments. Then, Section \ref{sec:nominal_planning_setup} discusses the implementation details of our \textbf{novel} gradient-based iLQR planner, which leverages the Markovian latent space to enable high task success rates (cf. Table \ref{tab:reacher_ablation_combined} for the ablation study). Next, Section \ref{sec:nominal_planning_baselines} details the baseline implementations used to evaluate our nominal planner against. 
For the safety experiments, Section \ref{app:sls_squared_implement} provides details on how our robust planner, \method, is implemented, and Section \ref{app:safety_baselines} provides details on the safety baselines we compare our approach to. Lastly, Section \ref{app:hardware_implementation} discusses our real-world hardware experiment.

\subsection{Experiment Setup}\label{app:exp_setup}

We evaluate on four continuous-control benchmarks spanning low-dimensional reaching,
rigid-object manipulation, planar pushing, and deformable-object manipulation.
For each task, we train from expert demonstrations and evaluate planning using
task-specific goal-reaching criteria.

\paragraph{Reacher.}
Reacher is the hard variant of the DeepMind Control Suite Reacher task, where a
two-joint planar arm must reach a target position. We render $224 \times 224$ RGB
observations with the target hidden from the image observation. The dataset
contains 24,650 expert episodes collected with a Soft Actor-Critic policy.
Episodes are collected at 50 Hz for up to 100 control steps and terminate early
when the end effector reaches within $0.03$ of the target.

\paragraph{Cube.}
Cube is the OGBench \texttt{cube-single-v0} task, where a robot arm must grasp and
move a cube to a target 3D position and yaw. We use front-camera pixel
observations and continuous 5D actions. The dataset contains 11,227 oracle-planner
episodes recorded at 20 Hz for up to 100 steps. Episodes terminate when the cube
is within $0.04$ m of the target position and within $0.2$ rad of the target yaw.

\paragraph{Push-T.}
Push-T is a 2D planar manipulation task where a circular pusher moves a T-shaped
block to a target pose. We use \texttt{gym-pusht/PushT-v0} and store RGB
observations, pusher positions, and continuous 2D actions. The dataset contains
25,000 expert episodes: 20,000 nominal diffusion-policy rollouts and 5,000
edge-case rollouts. Episodes contain up to 167 recorded steps.

\paragraph{Rope.}
Rope is a bimanual MuJoCo manipulation environment where two robot arms control
the endpoints of a deformable rope. The task exposes the model to nonlinear rope
dynamics by moving the rope attachments along randomized task-space trajectories.
We collect 15,000 episodes using a cubic-spline policy with six randomized
waypoints. Episodes are recorded at 20 Hz, average approximately 80 steps, and
are rendered as $224 \times 224$ RGB observations.

\paragraph{Planning success criteria.}
For Reacher, success is measured by the wrapped joint-position distance between
the current and goal configurations. The primary threshold is $0.1$ in qpos
distance, and we also report a stricter full-observation threshold of $0.05$ as a
diagnostic. For Cube, success requires the cube position to be within $0.04$ m of
the goal position. For Push-T, success requires the T-block to be within $20$
pixels of the goal position with yaw error at most $0.35$ rad. For Rope, success
requires the task-target distance to reach the dataset goal tolerance of
$0.05$.

\subsection{Nominal Planning Setup}\label{sec:nominal_planning_setup}
The nominal MPC planners optimize in the learned Markov latent state.  For Reacher, OGBench-Cube,
and Rope, the nominal planner is receding-horizon iLQR with horizon $T=15$, but the quadratic cost
matrices are task-specific (where weights scale the identity matrix $I_n \in \mathbb{R}^{n \times n}$):

\begin{table}[h]
\centering
\small
\caption{\textbf{Reacher, OGBench-Cube, Rope: }Nominal iLQR planning horizons and cost weights.}
\begin{tabular}{lcccc}
\toprule
Task & $T$ & $Q_{T}$ & $Q$ & $R$ \\
\midrule
Reacher & $15$ & $5.0I_{10}$ & $0.05I_{10}$ & $0.1I_2$ \\
OGBench-Cube & $15$ & $15.0I_{24}$ & $0.05I_{24}$ & $0.5I_5$ \\
Rope & $15$ & $15.0I_{24}$ & $0.005I_{24}$ & $0.01I_3$ \\
\bottomrule
\end{tabular}
\label{tab:ilqr_nominal_weights}
\end{table}

For these iLQR-only planners, the rollout is generated by the learned one-step dynamics
$\zeta_{t+1}=f_{\phi}(\zeta_t,v_t)$, and the open-loop control sequence
$\mathbf{v}$ is optimized with the finite-horizon cost
\begin{equation}
J_{\mathrm{iLQR}}(\bm{\zeta},\mathbf{v})
=
\sum_{t=0}^{T-1}
\left(
\|\zeta_t-\zeta_g\|_{Q}^2
+
\|v_t\|_{R}^2
\right)
+
\|\zeta_T-\zeta_g\|_{Q_T}^2,
\end{equation}
where $\zeta_g$ is the goal latent state. In Reacher, the same expression is used
but the cost is computed only on the zeroth-order latent coordinates (i.e., $z$), so the cost does not directly
penalize the finite-difference components of the Markov state. At each environment step, iLQR
linearizes the learned dynamics along the current rollout, performs a Riccati-style backward pass to
compute local feedback/feedforward updates, line-searches the updated control sequence, executes the
first action, and replans from the newly observed image.

For Push-T, the nominal planner uses MPPI to produce a long-horizon warm-start and then uses a
short-horizon iLQR tracker to refine the first segment. The MPPI horizon is $H^{\mathrm{mppi}}=45$, while the
iLQR tracking horizon is $T=15$. The MPPI cost matrices are
$Q_{T}^{\mathrm{mppi}}=10.0I_{144}$,
$Q^{\mathrm{mppi}}=0.05I_{144}$, and
$R^{\mathrm{mppi}}=0.01I_{2}$; the iLQR tracking matrices are
$Q_{T}=10.0I_{144}$,
$Q=1.0I_{144}$, and
$R=0.01I_{2}$. MPPI samples noisy control sequences, rolls them out
through the same latent dynamics, and scores each sampled trajectory with
\begin{equation}\label{eq:nom_mppi_cost}
J_{\mathrm{MPPI}}(\bm{\zeta},\mathbf{v})
=
\sum_{t=0}^{H^{\mathrm{mppi}}-1}
\left(
\|\zeta_t-\zeta_g\|_{Q^{\mathrm{mppi}}}^2
+
\|v_t\|_{R^{\mathrm{mppi}}}^2
\right)
+
\|\zeta_{H^{\mathrm{mppi}}}-\zeta_g\|_{Q_T^{\mathrm{mppi}}}^2.
\end{equation}
The sampled costs are converted to rewards by negation and weighted by the MPPI reward temperature
when updating the nominal control sequence. The iLQR tracker then optimizes over the shorter
horizon $T=15$ against the MPPI reference trajectory $\bar{\bm{\zeta}}$ rather than a fixed goal:
\begin{equation}
J_{\mathrm{track}}(\bm{\zeta},\mathbf{v})
=
\sum_{t=0}^{T-1}
\left(
\|\zeta_t-\bar{\zeta}_t\|_{Q}^2
+
\|v_t\|_{R}^2
\right)
+
\|\zeta_T-\bar{\zeta}_T\|_{Q_T}^2.
\end{equation}
Thus the MPPI stage adds a sampling-based global search over long control sequences, while the iLQR
stage locally smooths and tracks the selected latent trajectory before receding-horizon execution.

\subsection{World Model Baseline Implementation}\label{sec:nominal_planning_baselines}

\paragraph{Training and evaluation protocol.}
This section describes the implementation details for the world-model baselines and our planning pipeline. Across all tasks, we follow the training and evaluation workflow of the \texttt{stable-worldmodel} codebase whenever possible. We compare against PLDM, DINO-WM, and Le-WM. Unless otherwise specified, each baseline is trained for $10$ epochs following the Le-WM training protocol~\cite{lewm}, where this training length was reported to provide strong planning performance. When task-compatible pretrained or reusable checkpoints are available, we initialize from those checkpoints and fine-tune on our dataset. In particular, due to the size of the Push-T dataset and the cost of training from scratch, the Le-WM model for Push-T is fine-tuned from a pretrained Push-T object checkpoint. For the customized Rope task, no pretrained checkpoints are available, so all world models are trained from scratch using the same offline dataset.

\paragraph{Shared preprocessing and architecture.}
All baselines are trained from image observations resized to $224 \times 224$ with patch size $14$. We use a history length of $3$ and a one-step prediction horizon. During planning, the learned one-step latent dynamics are rolled out autoregressively to evaluate multi-step action sequences. The dataset is split into $90\%$ training and $10\%$ validation, and we apply gradient clipping with threshold $1.0$. Unless otherwise stated, models are trained with batch size $128$. We use frame skip $5$ for Reacher, Cube, and Push-T, and frame skip $1$ for Rope.

\paragraph{PLDM.}
For PLDM, we train with AdamW using learning rate $5\times 10^{-5}$ and weight decay $10^{-3}$. The encoder is a ViT-Tiny model with embedding dimension $192$. The predictor has depth $6$, $16$ attention heads, MLP dimension $2048$, head dimension $64$, and dropout $0.1$. Proprioceptive inputs are disabled. The active loss weights are $18$ for the standard-deviation loss, $0.7$ for the temporal standard-deviation loss, $12$ for the covariance loss, and $0.2$ for the temporal alignment loss. 

\paragraph{DINO-WM.}
For DINO-WM, we use a frozen DINOv2-Small visual backbone and train the latent predictor with AdamW using learning rate $5\times 10^{-4}$ and no weight decay. The predictor uses the same architecture as PLDM: depth $6$, $16$ attention heads, MLP dimension $2048$, head dimension $64$, and dropout $0.1$. Actions are encoded with dimension $10$. Reacher, Cube, and Rope are trained without proprioceptive encoding, while Push-T uses proprioceptive encoding.

\paragraph{Le-WM.}
For Le-WM, we train with AdamW using learning rate $5\times 10^{-5}$ and weight decay $10^{-3}$. The model uses a ViT-Tiny encoder with image size $224$, patch size $14$, history size $3$, prediction horizon $1$, and embedding dimension $192$. The latent predictor uses the same configuration as the other baselines: depth $6$, $16$ attention heads, MLP dimension $2048$, head dimension $64$, and dropout $0.1$. For Rope, Le-WM is trained from scratch with batch size $128$, frame skip $1$, and SigReg regularization with weight $0.09$, $17$ knots, and $1024$ random projections.

For Push-T, we fine-tune the pretrained Le-WM Push-T object checkpoint on the combined Push-T dataset. The Push-T model uses frame skip $5$, action dimension $2$, and effective action dimension $10$. We train for $10$ epochs with batch size $110$, bfloat16 mixed precision, gradient clipping at $1.0$, and the Le-WM latent prediction objective. During fine-tuning, we use SigReg regularization with weight $0.005$, $17$ knots, and $1024$ random projections, together with temporal straightening loss with weight $0.01$.

\paragraph{Planning evaluation.}
For nominal world-model planning baselines, PLDM, DINO-WM, and Le-WM use the standard CEM planner provided by the \texttt{stable-worldmodel} implementation. CEM optimizes action sequences in the learned latent space and executes the selected actions in a receding-horizon manner. For our method, we use the same learned latent dynamics as the planning model but replace CEM with our gradient-based latent planner. In the robust setting, this planner is augmented with the SLS-based tube propagation and conformalized latent constraints described in the main text. This setup ensures that baseline performance reflects the default planning pipeline associated with each learned world model, while our method differs only in the planning and robustness layer built on top of the learned latent dynamics.

\clearpage
\newpage 
\subsection{Implementing \method~(Problem \eqref{eq:cp_sls_mpc}) }\label{app:sls_squared_implement}
Below we discuss our implementation of Problem \eqref{eq:cp_sls_mpc}. For self-containment, we restate the problem: 
\begin{subequations}
\begin{align}
\min_{\bm{\zeta},\mathbf{v},\Phi^s,\Phi^u}\quad
    & J(\bm{\zeta},\mathbf{v}) + J_f(\zeta_T) + H(\Phi^s,\Phi^u)
    \tag{\ref{eq:sls_objective}}\\
\mathrm{s.t.}\quad
    & \zeta_{k+1}=f_\phi(\zeta_k,v_k),\quad \zeta_0=\bar s_0,
      \qquad k=0,\ldots,T-1,
    \tag{\ref{eq:sls_dyn}}\\
    & \Phi^s_{k+1,j}=A_k\Phi^s_{k,j}+B_k\Phi^u_{k,j},
      \qquad 0\le j<k,\ k=0,\ldots,T-1,
    \tag{\ref{eq:sls_slp}}\\
    & \Phi^s_{j+1,j}=E_j,
      \qquad j=0,\ldots,T-1,
    \tag{\ref{eq:sls_slp_init}}\\
    & g_i(\zeta_k,v_k)+b_i+
      \tau_{i,k}
      \le 0,
      \qquad i=1,\ldots,n_c,\ k=0,\ldots,T,
    \tag{\ref{eq:sls_tightenings}}\\
    & \mathcal{R}_k^s\subseteq\mathcal{I},
      \qquad k=0,\ldots,T.
    \tag{\ref{eq:sls_indist}}
\end{align}
\end{subequations}

In Section \ref{sec:problem_objective} we discuss the cost and reference trajectory generation used for the objective in \eqref{eq:sls_objective}. Next, in Section \ref{sec:calibrate_error} we discuss how the model error is calibrated. The calibrated error is used to initialize the system level constraints (SLC) in \eqref{eq:sls_slp}-\eqref{eq:sls_slp_init}. Details on the obstacles and how we train a latent classifier and conformalize it are available in Section \ref{sec:latent_classifier}. In the optimization problem, the classifier is linearized and tightened in \eqref{eq:sls_tightenings}. Lastly, we discuss the in-domain ellipsoid constraint, \eqref{eq:sls_indist}, and how it is formed in Section \ref{sec:conformalized_latent}.

\subsubsection{Defining the Objective}\label{sec:problem_objective}
To implement Problem \eqref{eq:cp_sls_mpc}, we follow \cite{fang2026safe, srinivasan2026safety} and use a real-time-iteration (RTI) scheme which solves a single SQP (instead of solving several successive SQPs until convergence) to enable faster solve times. To avoid ill-conditioned problems and improve numerical stability, we use the MPPI algorithm \cite{williams2017model} to warm-start the MPC solve at each time step. Below, we describe the MPPI warm-start objective, the nominal SLS objective, the associated cost weights, and the modifications used for the Constraint-iLQR baselines.

Let $\zeta_t \in \mathbb{R}^{n_s}$ denote the latent state and $\zeta_g$ the goal latent state. MPPI first generates a reference trajectory $\bar{\bm{\zeta}} := \{\bar\zeta_k\}_{k=0}^{T}$ and $\bar{\mathbf{v}}:=\{\bar{v}_k\}_{k=0}^{T-1}$ by minimizing the sampling objective
\begin{equation}
J_{\mathrm{MPPI}}(\bm{\zeta}, \mathbf{v}) =
\sum_{t=0}^{H^{\mathrm{mppi}}-1}
\left(
\|\zeta_t-\zeta_g\|_{Q^{\mathrm{mppi}}}^{2}
+
\|v_t\|_{R^{\mathrm{mppi}}}^{2}
+
J_{\mathrm{soft}}(\zeta_t)
\right)
+
\|\zeta_{H^{\mathrm{mppi}}}-\zeta_g\|_{Q^{\mathrm{mppi}}_{T}}^{2},
\label{eq:mppi_warmstart_cost}
\end{equation}
where $J_{\mathrm{soft}}$ denotes task-dependent soft penalties used during MPPI sampling. The MPPI horizon $H^\mathrm{mppi}$ can be greater than $T$ (the tracking horizon used by SLS) to encourage smaller actions at each step.  Furthermore, we define $\|x\|_{M}^2 := x^{\top}Mx$ for $x \in \mathbb{R}^n$ and $M \in \mathbb{R}^{n \times n}$.
The nominal RTI-SQP/SLS objective then tracks the MPPI reference trajectory $\bar{\bm{\zeta}}$ and $\bar{\mathbf{v}}$ over the MPC horizon $T$:
\begin{equation}
J_{\mathrm{SLS}}(\bm{\zeta}, \mathbf{v}) =
\sum_{t=0}^{T-1}
\left(
\|\zeta_t-\bar{\zeta}_t\|_{Q}^{2}
+
\|v_t-\bar{v}_t\|_{R}^{2}
\right)
+
\|\zeta_T-\bar{\zeta}_T\|_{Q_T}^{2}.
\label{eq:sls_nominal_cost}
\end{equation}

Our cost matrices penalizing state tracking $Q_{T}, Q, Q_{T}^\text{mppi}, Q^\text{mppi} \succeq 0$ and control penalties $R, R^\text{mppi}\succ0$ are summarized in Table~\ref{tab:mppi_sls_weights}. While interpretability is minimal in the latent space, we leverage the Markovian structure to apply different penalties to position states (corresponding to the zeroth-order state) and velocity states (corresponding to the first-order state). Compactly, we use $I_n \in \mathbb{R}^{n\times n}$ as the $n$-dimensional identity matrix, and for $A \in \mathbb{R}^{n\times n}, B\in \mathbb{R}^{m\times m}$ we denote $\mathrm{blkdiag}(A,B)$ as the block diagonal matrix.  

\begin{table}[t]
\centering
\small
\caption{MPPI warm-start and nominal SLS tracking weights used in our experiments.}
\begin{tabular}{llccc}
\toprule
Task & Objective & Horizon & State Weights & Control Weight \\
\midrule
\multirow{3}{*}{Rope}
& MPPI
& $15$
& $Q^{\mathrm{mppi}}=\mathrm{blkdiag}(0.005I_{12}, 1.0I_{12})$
& $R^{\mathrm{mppi}}=1.0I_{3}$ \\
&
&
& $Q^{\mathrm{mppi}}_{T}=\mathrm{blkdiag}(10I_{12}, 1.0I_{12})$
& \\
\cmidrule(lr){2-5}
& SLS
& $10$
& $Q=0.1I_{24},\quad Q_T=5.0I_{24}$
& $R=0.2I_{3}$ \\
\midrule
\multirow{4}{*}{Reacher}
& MPPI
& $60$
& $Q^{\mathrm{mppi}}=\mathrm{blkdiag}(0.005I_{5}, 1.0I_{5})$
& $R^{\mathrm{mppi}}=0.01I_{2}$ \\
&
&
& $Q^{\mathrm{mppi}}_{T}=\mathrm{blkdiag}(100I_5, 1.0I_{5})$
& \\
\cmidrule(lr){2-5}
& SLS
& $5$
& $Q=\mathrm{blkdiag}(0.005I_5, 1.0I_{5})$
& $R=0.1I_2$ \\
&
&
& $Q_T=\mathrm{blkdiag}(5.0I_5, 0.9I_{5})$
& \\
\midrule
\multirow{4}{*}{OGBench-Cube}
& MPPI
& $10$
& $Q^{\mathrm{mppi}}=\mathrm{blkdiag}(0.05I_{12}, 0.10I_{12})$
& $R^{\mathrm{mppi}}=1.0I_5$ \\
&
&
& $Q^{\mathrm{mppi}}_{T}=\mathrm{blkdiag}(15I_{12}, 1.0I_{12})$
& \\
\cmidrule(lr){2-5}
& SLS
& $10$
& $Q=\mathrm{blkdiag}(0.05I_{12}, 0.05I_{12})$
& $R=0.5I_5$ \\
&
&
& $Q_T=\mathrm{blkdiag}(1.0I_{12}, 1.0I_{12})$
& \\
\bottomrule
\end{tabular}
\label{tab:mppi_sls_weights}
\end{table}

In \eqref{eq:mppi_warmstart_cost}, we also add a soft penalty, $J_\textrm{soft}$, defined as
\begin{equation}
J_{\mathrm{soft}}(\zeta)
=\lambda_{\mathrm{ID}}
\underbrace{\left[
\max\left\{
\rho_{\mathrm{ID}}(\zeta)-1,
0
\right\}
\right]^2}_{\text{Latent Ellipsoid Penalty}} + \lambda_{\mathrm{obs}}
\underbrace{\left[
\operatorname{softplus}
\bigl(
\eta  - c_\psi(\Pi_z \zeta)
\bigr)
\right]^2}_{\text{Obstacle Penalty}}
\label{eq:mppi_soft_latent_region}
\end{equation}
where $\rho_{\mathrm{ID}}(\zeta)$ is the conformal latent ellipsoid defined in \eqref{eq:in_domain} and $c_\psi(\Pi_z \zeta)$ is the obstacle classifier score evaluated on the zeroth-order latent component. 
$\lambda_\textrm{obs}$ and $\lambda_\mathrm{ID}$ are hyperparameters controlling the strength of the respective penalties.
The soft penalty in \eqref{eq:mppi_soft_latent_region} encourages the reference trajectory to satisfy the constraints. Notably, these references are not guaranteed to robustly satisfy the constraints, and thus SLS is still required to enforce robust constraint satisfaction. The remaining MPPI hyperparameters alongside the values of $\lambda_\textrm{obs}$ and $\lambda_\mathrm{ID}$ used for each task are summarized in Table~\ref{tab:mppi_hyperparameters}.

\begin{table}[t]
\centering
\small
\caption{MPPI hyperparameters and soft-penalty weights used for warm-start trajectory generation.}
\begin{tabular}{lccccccc}
\toprule
Task &
Samples &
Updates &
Reward Weight &
Noise Scale &
$\beta$ &
$\lambda_{\mathrm{ID}}$ &
$\lambda_{\mathrm{obs}}$ \\
\midrule
Rope &
2048 &
6 &
25 &
0.20 &
0.65 &
20 &
0 \\

Reacher &
512 &
5 &
20 &
0.15 &
0.70 &
0 &
1000 \\

OGBench-Cube &
512 &
5 &
30 &
0.20 &
0.60 &
10 &
100 \\
\bottomrule
\end{tabular}
\label{tab:mppi_hyperparameters}
\end{table}

\subsubsection{Calibrated Error Bounds} \label{sec:calibrate_error}
To calibrate the error bounds we split our dataset $\mathcal{D}_\textrm{cal}$ in half (i.e., $\mathcal{D}_\textrm{cal}^{(1)}$ and $\mathcal{D}_\textrm{cal}^{(2)}$ such that $|\mathcal{D}_\textrm{cal}^{(1)}| = |\mathcal{D}_\textrm{cal}^{(2)}|$. 
For our experiments, $\mathcal{D}_\textrm{cal}$ is constructed using the same expert rollouts used to train the models, as the transitions encountered in $\mathcal{D}_\textrm{cal}$ are distributionally similar to the transitions encountered during planning time, preserving the exchangeability assumption. For Reacher, Rope, and OGBench-Cube, $|\mathcal{D}_\textrm{cal}|=200,000$, $|\mathcal{D}_\textrm{cal}|=80,000$, and $|\mathcal{D}_\textrm{cal}|=98,800$ respectively.
The first partition, $\mathcal{D}_\textrm{cal}^{(1)}$, is used to learn the error covariance, $\Sigma_{\psi}(s,u)$, structure by minimizing the MGNLL loss \eqref{eq:MGNLL}, while the second half is used to perform the conformal calibration described in Section \ref{sec:cp}. For the Reacher task, we trained $\Sigma_{\psi}(s,u)$ as a neural network with 2 hidden layers (128 neurons per hidden layer) at a learning rate of 0.0003 using a cosine annealing learning rate scheduler to stabilize training. The model directly outputs the Cholesky factor to avoid computing it in real time. For the Rope and OGBench-Cube tasks, we defined $\Sigma_{\psi}(s,u):= \Sigma_\textrm{cov}$ as a fixed parameter. In this case, the loss is minimized by computing $\Sigma_\textrm{cov} = \frac{1}{|\mathcal{D}_\textrm{cal}^{(1)}|}\sum_{(s_i,u_i,s_i^+)\in\mathcal{D}_\textrm{cal}^{(1)}}(s_i^+ -f_\phi(s_i,u_i))(s_i^+ -f_\phi(s_i,u_i))^{\top}$ as the sample error covariance. For the conformal calibration we set $\alpha=0.1, 0.4,~\text{and}~0.1$ for Reacher, Rope, and OGBench-Cube respectively.

\subsubsection{Training and Conformalizing the Constraint Classifier}\label{sec:latent_classifier}
For each constrained task, we train a latent-space classifier from a balanced dataset of rendered
constraint-violating and non-violating observations.  Let
\[
\mathcal{D}_{\mathrm{obs}}=\{(o_i,y_i)\}_{i=1}^{N},\qquad y_i\in\{0,1\},
\]
where, following the notation in the main text, $y_i=0$ denotes a constraint-violating observation
and $y_i=1$ denotes a non-violating observation.  The task-specific geometric rule used to assign
$y_i$ is described below.  In the implementation, the stored source label is the complementary
indicator $\tilde b_i=1-y_i$, so $\tilde b_i=1$ marks the violating class used for conformal calibration.  Each
image is encoded with the frozen world-model encoder
$z_i=\mathrm{enc}_{\theta}(o_i)\in\mathbb{R}^{n_z}$, and the latent features are standardized using
the training-split mean and standard deviation.  Let
$\mathcal{I}_{\mathrm{tr}}$, $\mathcal{I}_{\mathrm{val}}$, and $\mathcal{I}_{\mathrm{cal}}$ denote
the index sets for the training, validation, and conformal calibration splits of
$\mathcal{D}_{\mathrm{obs}}$, respectively.  We then train a small MLP score model
$c_{\psi}:\mathbb{R}^{n_z}\to\mathbb{R}$ with the convention that larger scores indicate safer
states.  Equivalently, the signed labels are
\[
\tilde y_i =
\begin{cases}
-1, & y_i=0,\\
+1, & y_i=1,
\end{cases}
\]
and the classifier is optimized with the hinge loss
\[
\mathcal{L}_{\mathrm{obs}}(\psi)
=
\frac{1}{N_{\mathrm{tr}}}
\sum_{i\in\mathcal{I}_{\mathrm{tr}}}
\max\bigl(0,m-\tilde y_i c_{\psi}(z_i)\bigr).
\]
Thus, before conformal calibration, $c_{\psi}(z)>0$ is classified as non-violating and
$c_{\psi}(z)\le 0$ is classified as violating.

We calibrate the classifier using only held-out violating examples.  Let
$\mathcal{I}_{\mathrm{cal}}^{-}=\{i\in\mathcal{I}_{\mathrm{cal}}:y_i=0\}$ and define the
nonconformity score
\[
r_i=\max\{0,c_{\psi}(z_i)\},\qquad i\in\mathcal{I}_{\mathrm{cal}}^{-}.
\]
This score is zero for a violating calibration point that is already on the violating side of the
nominal classifier and positive only when the point is incorrectly placed on the safe side.  For
miscoverage level $\delta$, we choose
\[
\eta =
\operatorname{Quantile}_{\lceil (n_{\mathrm{cal}}^{-}+1)(1-\delta)\rceil}
\left(\{r_i:i\in\mathcal{I}_{\mathrm{cal}}^{-}\}\cup\{\infty\}\right),
\]
where $n_{\mathrm{cal}}^{-}=|\mathcal{I}_{\mathrm{cal}}^{-}|$.  The conformalized violating set is
\[
\mathcal{O}_{\mathrm{conf}}
=
\{z: c_{\psi}(z)< \eta\}.
\]
Under the exchangeability assumption for future violating observations and the calibration split,
this construction covers a future violating latent with probability at least $1-\delta$.

To match the SLS notation that is used in the main text, we convert the classifier score into a constraint
function of the following form:
\[
g_{\psi}(s)=-c_{\psi}(z),
\qquad z=\Pi_z s,
\qquad
\mathcal{S}_{\mathrm{safe}}
=
\{s:g_{\psi}(s)\le -\eta\}
=
\{s:-c_{\psi}(\Pi_z s)\le -\eta\}.
\]
Here $\Pi_z$ extracts the zeroth-order latent embedding from the Markov latent state
$s=[z,\Delta z,\ldots,\Delta^K z]$.  This is the sign conversion used by the SLS planner: the raw
network score $c_{\psi}$ is larger for safe states, while the MPC constraint is written in the paper's
standard form $g_{\psi}(s)\le 0$.  The same threshold also appears in the MPPI warm-start soft
penalty as a smooth relaxation of $\eta-c_{\psi}(z)\le 0$. 

\paragraph{Translating Constraints for SLS:}
Using the obstacle constraint we derived, it can be plugged into SLS and robustly satisfied by tightening the constraint. For nominal constraint satisfaction at step $k$, in the trajectory ($0\le k\le T$) we require, 
$\bar{g}_\psi(\zeta_k, v_k) \le -\eta$, where we define $\bar{g}_\psi(\zeta_k, v_k) = g(z_k)$ such that $z_k$ is the zeroth-order component of the Markov state $\zeta_k$.  We then apply the constraint tightening as a function of the linearized constraint and tube (as discussed in Appendix \ref{app:sls}) to robustly satisfy the constraint for all Markov states, $s_k \in \mathcal{R}^s_k$ and control inputs  $u_k \in \mathcal{R}^u_k$ in the tube. In our implementation we also have box constraints for the control input (i.e., $u_k \le u_\textrm{max} \texttt{ and } -u_k \le -u_\textrm{min},\forall ~0 \le k \le T-1$). Accounting for both upper and lower bound control constraints, this results in a total of $2n_u + 1 =n_c$ constraints in our problem. For succinctness, we collect all the constraints for each time step as follows:
\begin{equation}
    g_i(\zeta_k,v_k)+b_i+
      \tau_{i,k}
      \le 0,
      \qquad i=1,\ldots,n_c,\ k=0,\ldots,T,
\end{equation}
where $g_i(\zeta_k,v_k)$ is the constraint function, $b_i$ is the nominal bound (adjusted by $\eta$ for conformalized obstacle constraint indices), and $\tau_{i,k}$ is the constraint tightening.

\paragraph{Task-specific label rules and classifier sizes} \mbox{}

\textbf{Reacher:}
For Reacher, the violating set is the joint-space box
\[
\mathcal{Q}_{\mathrm{obs}}
=
\{q=(q_1,q_2):0\le q_1\le 3.1415,\ -2.88\le q_2\le -2.45\}.
\]
This region corresponds to configurations where the second link folds back very close to the first
link.  In the simulator this is a pose/joint-space obstacle constraint, while in planning it is enforced
as a learned latent-space constraint from rendered images.  We sample configurations inside
$\mathcal{Q}_{\mathrm{obs}}$ as violating examples and configurations outside the box as
non-violating examples, then render each configuration from the fixed planner camera to obtain the
balanced classifier dataset.  The Reacher classifier uses a $5$-dimensional latent input and an MLP
with one hidden layer of width $6$ followed by a scalar output, using GELU nonlinearities.

\textbf{OGBench:}
For Cube, the classifier enforces a height constraint on the grasped cube/effector state.  We
synthesize grasped cube configurations, render front-camera observations with the target hidden,
and label an observation as violating when the measured gripper height is above the configured
threshold
\[
h_{\mathrm{grip}}>h_{\max},\qquad h_{\max}=0.09.
\]
Samples with $h_{\mathrm{grip}}\le 0.09$ are treated as non-violating.  The sampled cube positions
use the task-space bounds $x\in[0.30,0.50]$, $y\in[-0.25,0.25]$, and $z\in[0.02,0.30]$.  If
precomputed dataset splits are unavailable, the script creates stratified source-training and
calibration splits before carving out a validation subset.  The Cube height classifier uses a
$12$-dimensional latent input and a linear score model
$c_{\psi}(z)=w_{\psi}^{\top}z+b_{\psi}$, i.e., a single learned separating vector and bias.

\textbf{Rope:}
For Rope, each sampled task state is
$x=(r,h,w)$, where $r$ is the endpoint reach coordinate, $h$ is the endpoint height, and $w$ is the
rope width.  The task-space bounds are
\[
r\in[-0.10,0.30],\qquad h\in[1.16,1.35],\qquad w\in[0.20,0.75].
\]
The obstacle is a speedbump-shaped barrier in the $(r,h)$ task-space slice, placed on the tabletop
over the active reach interval $[0.05,0.15]$.  Formally, it is the upper half of an ellipse with
center reach $r_c=0.10$, half-width $a=0.05$, base height $z_{\mathrm{base}}=0.75$, and peak
height $z_{\mathrm{peak}}=0.91$:
\[
b(r)=z_{\mathrm{base}}+
(z_{\mathrm{peak}}-z_{\mathrm{base}})
\sqrt{\max\left\{0,1-\left(\frac{r-r_c}{a}\right)^2\right\}}.
\]
Thus the bump rises $0.16$ above the table at its midpoint and returns to the table height at
$r=0.05$ and $r=0.15$.  Because the rope sags as a function of width, we estimate a sag profile
$s(w)$ from simulated proxy-rope configurations and use
\[
\ell(x)=h-s(w)
\]
as the approximate lowest rope height.  A rope state is labeled violating when
\[
r\in[0.05,0.15]\quad\text{and}\quad \ell(x)\le b(r),
\]
and non-violating otherwise.  This gives a formal task-space obstacle that visually behaves like a
smooth speedbump the rope must clear; the classifier then learns the corresponding latent-space
constraint from fixed-camera renderings.  The Rope classifier uses a $12$-dimensional latent input
and an MLP with one hidden layer of width $12$ followed by a scalar output, using GELU
nonlinearities.

\subsubsection{Calibrating the Latent Support}\label{sec:conformalized_latent}
To further validate the exchangeability assumption at planning time, we estimate and conformalize the support of our calibration dataset. SIGReg loss \eqref{eq:sigreg} structures the latent space such that the raw latent vectors follow an isotropic multivariate Gaussian distribution with unit variance. Thus, the Markovian latent vector also inherits the Gaussian structure, but is not guaranteed to be isotropic due to correlation between successive latent states. Hence, we learn both the shape and size of the ellipsoidal geometry characterizing the support of our calibration dataset, and constrain our planner to remain within it.
We refer to this constraint as the in-distribution (ID) latent ellipsoid. 
To obtain this constraint we use the same calibration dataset, $\mathcal{D}_\textrm{cal}$ and partitions, $\mathcal{D}_\textrm{cal}^{(1)}/\mathcal{D}_\textrm{cal}^{(2)}$from Section \ref{sec:calibrate_error}. We compute the sample covariance of the states, $\Sigma_{ID}$, using states, $s_i$, in $\mathcal{D}_\textrm{cal}^{(1)}$, and fix the mean $\mu = 0$. We then compute the Mahalanobis distance $d_i = s_i^{\top}\Sigma_\textrm{ID}^{-1}s_i$ for each state, $s_i \in \mathcal{D}_\textrm{cal}^{(2)}$. Lastly, we compute the $(1-\alpha_\textrm{ID})$-quantile $q_\textrm{ID}$ of distances $d_i$ to form the boundary of our support, as defined in \eqref{eq:in_domain}. Formally, this provides a boundary between in-distribution points which we assume to be $(1-\alpha_\textrm{ID} )\cdot100\%$ and outlier points corresponding to $\alpha_\textrm{ID}\cdot100\%$ of points in the calibration dataset, $\mathcal{D}_\textrm{cal}^{(2)}$. For each of our experiments we set $\alpha_\textrm{ID}=0.1$.

\subsection{Safety Baseline Implementation}\label{app:safety_baselines}

\paragraph{HJ-filtered safety baseline.}
We implement the HJ-filtered baseline as a latent safety shield applied on top of a nominal planner. We use the same encoder and latent dynamics notation as in the main method: an observation $o_t$ is encoded as
\[
    z_t=\mathrm{enc}_\theta(o_t),
\]
and the planner state is the Markov latent state
\[
    s_t^z =
    [z_t,\Delta z_t,\ldots,\Delta^K z_t].
\]
Actions are represented in the normalized action coordinates used by the learned world model:
\[
    \bar u =
    \mathrm{clip}\left(
    \frac{u-\mu_u}{\sigma_u},
    u_{\min},
    u_{\max}
    \right),
    \qquad
    u_{\min}=-2,\quad u_{\max}=2 .
\]
The HJ baseline rolls out the learned one-step latent dynamics
\[
    \hat s_{t+1}^z = f_\phi(s_t^z,\bar u_t).
\]

The HJ baseline uses its own obstacle classifier convention, inherited from the HJ-filtered baseline implementation, which is distinct from the constraint classifier notation used by our SLS planner in the main body. To avoid overloading $g_\psi$, we denote the HJ classifier score by $c_{\omega}$. Its signed safety margin is
\[
    \ell_{\mathrm{HJ}}(s)
    =
    c_{\omega}(s)-\tau_{\mathrm{HJ}},
\]
where $\tau_{\mathrm{HJ}}$ is the HJ classifier decision threshold, calibrated separately from the SLS constraint margin $\eta$. In our experiments, $\tau_{\mathrm{HJ}}$ is selected using the same held-out calibration split used for conformal safety calibration. Under this baseline convention, positive values indicate states classified as safe and non-positive values indicate states classified as unsafe. Depending on the experiment configuration, the raw margin is optionally transformed by a monotone squashing function,
\[
    \ell_{\mathrm{HJ}}(s)
    \leftarrow
    \tanh(\ell_{\mathrm{HJ}}(s))
    \quad \text{or} \quad
    \ell_{\mathrm{HJ}}(s)
    \leftarrow
    \tanh(2\ell_{\mathrm{HJ}}(s)).
\]

We train a PyHJ/Avoid-DDPG recovery policy $\pi_{\mathrm{HJ}}$ and critic $V_{\mathrm{HJ}}$ in the learned latent dynamics. Episodes are initialized from cached latent states. For a sampled action $\bar u_t$, the imagined next latent state is
\[
    \tilde s_{t+1}^z=f_\phi(s_t^z,\bar u_t),
\]
and the transition signal is the next-state safety margin:
\[
    r_t=\ell_{\mathrm{HJ}}(\tilde s_{t+1}^z).
\]
An episode terminates when $\ell_{\mathrm{HJ}}(\tilde s_{t+1}^z)\le0$. The learned critic therefore estimates recoverability under the latent recovery policy. At deployment, the HJ barrier score is
\[
    B_{\mathrm{HJ}}(s)
    =
    \min\{\ell_{\mathrm{HJ}}(s),V_{\mathrm{HJ}}(s)\}.
\]

Given a nominal raw action $u_t^{\mathrm{nom}}$ from the base planner, the filter first normalizes it to $\bar u_t^{\mathrm{nom}}$ and predicts
\[
    \hat s_{t+1}^{z,\mathrm{nom}}
    =
    f_\phi(s_t^z,\bar u_t^{\mathrm{nom}}).
\]
The nominal action is executed only if the current state is classified as safe and the predicted next state is accepted by the learned barrier. Otherwise, the filter replaces it with the HJ recovery action:
\[
    \bar u_t =
    \begin{cases}
    \pi_{\mathrm{HJ}}(s_t^z), &
    \text{if } \ell_{\mathrm{HJ}}(s_t^z)\le0
    \text{ or }
    B_{\mathrm{HJ}}(\hat s_{t+1}^{z,\mathrm{nom}})\le\epsilon,\\
    \bar u_t^{\mathrm{nom}}, &
    \text{otherwise}.
    \end{cases}
\]
We use $\epsilon=0$ unless otherwise specified. The selected normalized action is then unnormalized before execution:
\[
    u_t=\bar u_t\sigma_u+\mu_u .
\]

\paragraph{Latent Policy Barrier baseline.}
We also evaluate an adapted Latent Policy Barrier (LPB) baseline, which incorporates a nearest-neighbor in-distribution barrier as a regularizer in the nominal iLQR planner. First, we build a bank of safe Markov latent states from the training dataset. Each image is encoded as $z_t=\mathrm{enc}_\theta(o_t)$ and converted into the same Markov latent state $s_t^z$ used by the planner. Unsafe states are removed using the analytical ground-truth obstacle check used to label obstacle data. Thus, the prototype bank contains only latent states whose corresponding environment configurations are collision-free under the geometry-based safety test.

The remaining safe states are whitened,
\[
    \tilde s =
    \frac{s-\mu_{\mathrm{LPB}}}{\sigma_{\mathrm{LPB}}},
\]
and $M$ prototype states $\mathcal P=\{p_j\}_{j=1}^M$ are sampled from the whitened safe set. We calibrate a distance threshold $\rho$ as a chosen quantile of nearest-prototype distances on held-out non-prototype safe states. At planning time, each predicted latent state is scored by
\[
    d_{\mathrm{LPB}}(s)
    =
    \min_{p_j\in\mathcal P}
    \left\|
    \frac{s-\mu_{\mathrm{LPB}}}{\sigma_{\mathrm{LPB}}}
    -
    p_j
    \right\|_2 .
\]
The LPB penalty is zero inside the calibrated safe tube and grows quadratically outside it:
\[
    c_{\mathrm{LPB}}(s)
    =
    \left[
    d_{\mathrm{LPB}}(s)-\alpha_\textrm{LPB}\rho
    \right]_+^2,
\]
where $\alpha_\textrm{LPB}$ is the threshold scale and $[\cdot]_+=\max(\cdot,0)$.

Unlike the HJ filter, LPB does not override actions after planning. Instead, it augments the iLQR trajectory objective. For a nominal latent trajectory $\bm{\zeta}=\{\zeta_k\}_{k=0}^{T}$ and normalized controls $\mathbf{v}=\{v_k\}_{k=0}^{T-1}$, the LPB-regularized objective is
\[
    J_{\mathrm{LPB}}
    =
    \sum_{k=0}^{T-1}
    \left(
    q_s\|\zeta_k-\zeta_g\|_2^2
    +
    q_u\|v_k\|_2^2
    \right)
    +
    q_f\|\zeta_T-\zeta_g\|_2^2
    +
    \lambda_{\mathrm{LPB}}
    \frac{1}{T}
    \sum_{k=0}^{T-1}
    c_{\mathrm{LPB}}(\zeta_k),
\]
where $\zeta_g$ is the goal Markov latent state. In our implementation, the LPB term is differentiable through the latent trajectory, and its gradient and Hessian are included in the iLQR backward pass. Thus, LPB biases the planner toward trajectories that remain close to the empirical distribution of safe latent states while still optimizing the original goal-reaching cost.

\paragraph{Constrained iLQR.} For our constrained iLQR implementation we solve Problem \eqref{eq:cp_sls_mpc} without constraints \eqref{eq:sls_slp}-\eqref{eq:sls_slp_init}, and remove the constraint tightening, $\tau_{i,k}$ in \eqref{eq:sls_tightenings}. This results in a constrained safety-aware problem, but does not require a trajectory to robustly satisfy constraints--unlike SLS. Furthermore we have two variants, one with the in-domain latent ellipsoid constraint and one without. Outside of these changes, both of these approaches still use the same parameters and configurations as the SLS implementation.

\subsection{Hardware Implementation}\label{app:hardware_implementation}
Our hardware experiment is the physical equivalent of our simulated bimanual rope task. 
\paragraph{Hardware:} The bimanual setup consists of two KUKA iiwa arms placed 1.5 meters apart, each with a Schunk end-effector. The rope is attached to the arms via two 3D-printed clips that are screwed into the Schunk end-effector. Figures \ref{fig:rope_real_nominal} and \ref{fig:rope_real_2_nominal} provide visualizations of the hardware setup. To capture images, we used an Intel RealSense D415 camera (and use only the RGB images). 

\paragraph{Data Collection:}
We parameterize the rope manipulation task with a three-dimensional task space: reach, the shared forward/backward displacement of the end effectors; height, their shared vertical position; and width, their lateral separation. Each task-space command is converted to bimanual joint targets using Drake \cite{drake} inverse kinematics, with the two grippers moved symmetrically about the rope centerline. Data is collected in shards. At the start of each shard, the arms return to a fixed home pose, which defines the nominal task origin. We then sample 25 random task-space knots within the workspace bounds, enforcing a minimum separation between consecutive knots. We linearly interpolate between each pair to produce a dense sequence of intermediate task targets. This waypoint-based trajectory gives broad workspace coverage while keeping each low-level motion small and feasible for hardware execution. We considered smoother spline trajectories, as used in the sibling simulation experiment, but use piecewise-linear waypoint segments on the real system because they make the commanded task deltas explicit, easier to validate for collision and speed limits, and easier to execute with settling between samples. For each interpolated target, the corresponding joint command is executed, the arms are allowed to settle, and an RGB image is captured. We record the planned task-space delta together with the measured end-effector state and executed task-space delta, producing image-action pairs for training the dynamics model.

\paragraph{Model Training:} 
We collect our data with a data structure similar to that of all our data collection scripts for simulation. This enables us to swap in our hardware data directly with our simulation data. All the model parameters are equivalent to the corresponding simulation experiment for nominal planning.   

\paragraph{Planning:} For nominal planning, we used MPPI to warm-start trajectory generation at each step and then iLQR to track it, similar to the Push-T approach described in Section \ref{sec:nominal_planning_setup}. Since the low-level controller requires a minimal actuation in order to move, we define
\begin{equation}
J_{\mathrm{act}}(\mathbf{v}) = \lambda_{\mathrm{act}} \sum_{t=0}^{T-1} \left[ \max\left( 0,\, \frac{\delta_{\min} - \|v_t\|_\infty}{\delta_{\min}} \right) \right]^2, 
\end{equation}
where $\delta_{\min}$ is the minimum actuation. Since the bimanual rope manipulation task has 3 actions controlling the distance between the two end-effectors (width), the height of the two end-effectors, and the depth of two end-effectors (reach), we only require one of the actions to meet the minimum actuation requirement. Thus, we compute the infinity norm of the nominal action $\|v_t\|_\infty$ to get the maximum value among the 3 actions and penalize it if it is below the minimum actuation threshold, $\delta_{\min}$. Thus, the cost used for the real-world experiment $J^\textrm{real}_\textrm{mppi}(\bm{\zeta}, \mathbf{v})$ is, 
\begin{equation}
    J^\textrm{real}_\textrm{MPPI}(\bm{\zeta}, \mathbf{v}) = \underbrace{J_\textrm{MPPI}(\bm{\zeta}, \mathbf{v})}_\textrm{Eq. \eqref{eq:nom_mppi_cost}} + J_{\mathrm{act}}(\mathbf{v}).
\end{equation}
Similar to the simulation experiment (for Push-T), we use iLQR to track the MPPI-generated reference trajectory, and provide the next control input to apply at each MPC step. The MPPI parameters and iLQR parameters are available in Table \ref{tab:real_world_planning_params}.

\begin{table}[h]
\centering
\small
\caption{MPPI and iLQR planning parameters for the real-world rope manipulation task.}
\begin{tabular}{llclc}
\toprule
Module & Parameter & Value & Parameter & Value \\
\midrule
\multirow{5}{*}{MPPI Warm-start} & Horizon ($H^{\mathrm{mppi}}$) & $60$ & Control Weight ($R^{\mathrm{mppi}}$) & $0.02I_{3}$ \\
& Samples & $512$ & Reward Weight & $25.0$ \\
& Updates & $6$ & Noise Scale & $0.2$ \\
& State Stage Weight ($Q^{\mathrm{mppi}}$) & $1.5I_{24}$ & Filter $\beta$ & $0.65$ \\
& State Terminal Weight ($Q_T^{\mathrm{mppi}}$) & $1000.0I_{24}$ & & \\
\midrule
Actuation Penalty & Minimum Threshold ($\delta_{\min}$) & $0.006$ & Penalty Weight ($\lambda_{\mathrm{act}}$) & $10.0$ \\
\midrule
\multirow{2}{*}{iLQR Tracking} & Horizon ($T$) & $10$ & State Terminal Weight ($Q_T$) & $50.0I_{24}$ \\
& State Stage Weight ($Q$) & $0.5I_{24}$ & Control Weight ($R$) & $0.02I_3$ \\
\bottomrule
\end{tabular}
\label{tab:real_world_planning_params}
\end{table}

\clearpage
\newpage 
\section{Extended Results}\label{app:extend_results}

In this section, we provide additional experimental results beyond those reported in the main text. Figures \ref{fig:pusht_nominal} - \ref{fig:rope_nominal} contain our nominal planning time-lapses where we visualize our goal-reaching performance. Figures \ref{fig:rope_real_nominal} - \ref{fig:rope_real_2_nominal} contain our nominal planning time-lapses for our bimanual rope manipulation hardware experiment and figure \ref{fig:rope_real_task_error} validates convergence of our planner to the goal configuration in the task space for real hardware. Figures \ref{fig:rope_safe_unsafe} - \ref{fig:cube_safe_unsafe} contain time-lapse comparisons of safe vs. unsafe plans demonstrating \method's ability to enforce safety. Figures \ref{fig:rope_tubes} - \ref{fig:cube_tubes} visualize the reachable tubes in the latent space. Lastly, we include the LPB safety baseline on Reacher, the LPB safety baseline on Cube, and the Le-WM nominal planning baseline on Push-T. 

\subsection{Nominal Trajectory Time-lapses}
\begin{figure}[H]
    \centering\vspace{-8pt}
    \includegraphics[width=1.0\linewidth, trim={0.0cm 0.0cm 0.0cm 0.0cm}, clip]{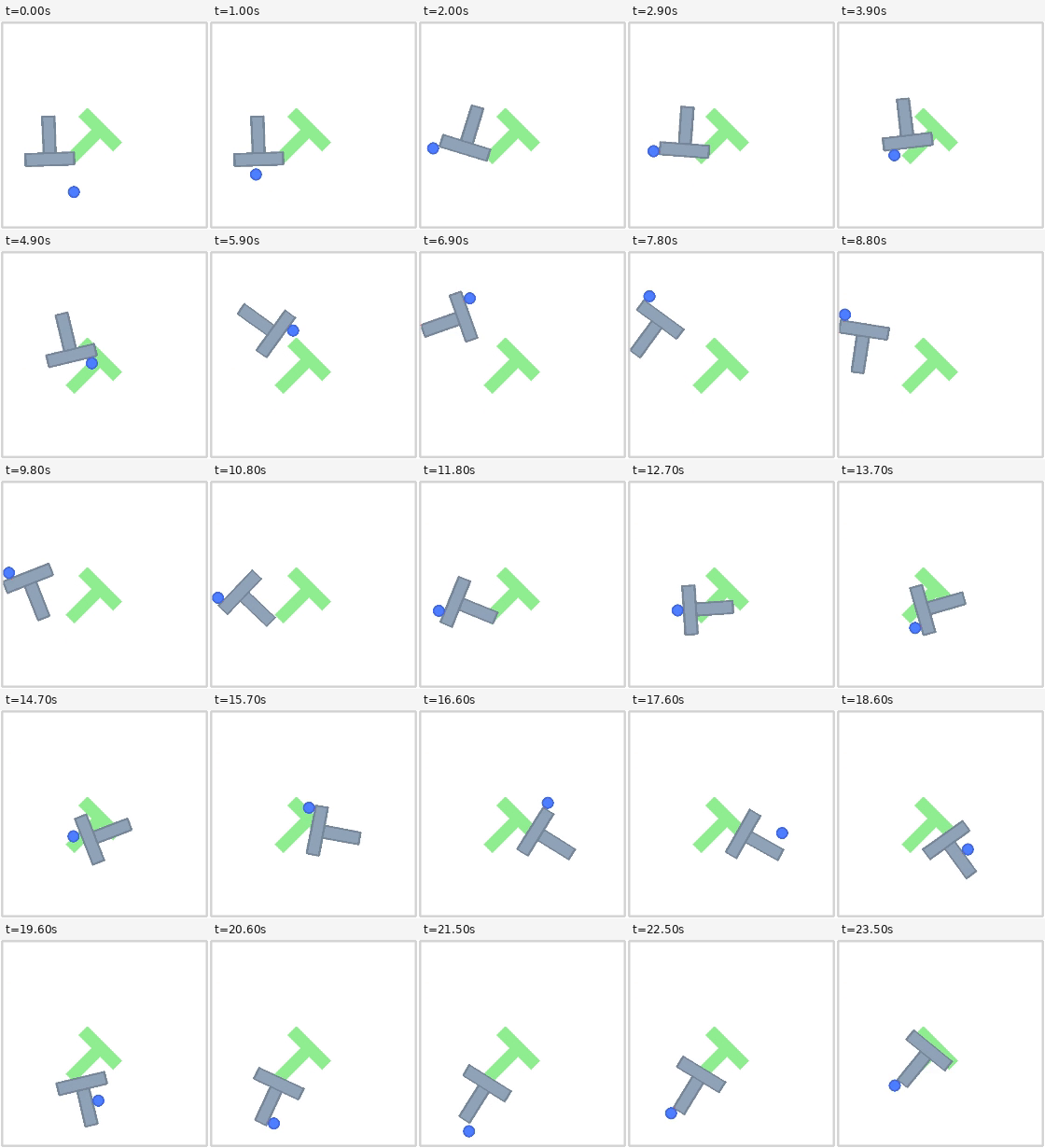}
    \vspace{-0.5em}
    \caption{\textbf{Push-T.} iLQR-based long-horizon planning, with MPPI warm-start. Our planner is able to successfully make and break contact with the T while executing complex maneuvers in order to push the T to the goal.}
    \label{fig:pusht_nominal}
\end{figure}
\clearpage

\begin{figure}[H]
    \centering\vspace{-8pt}
    \includegraphics[width=\linewidth, trim={0.0cm 0.0cm 0.0cm 0.0cm}, clip]{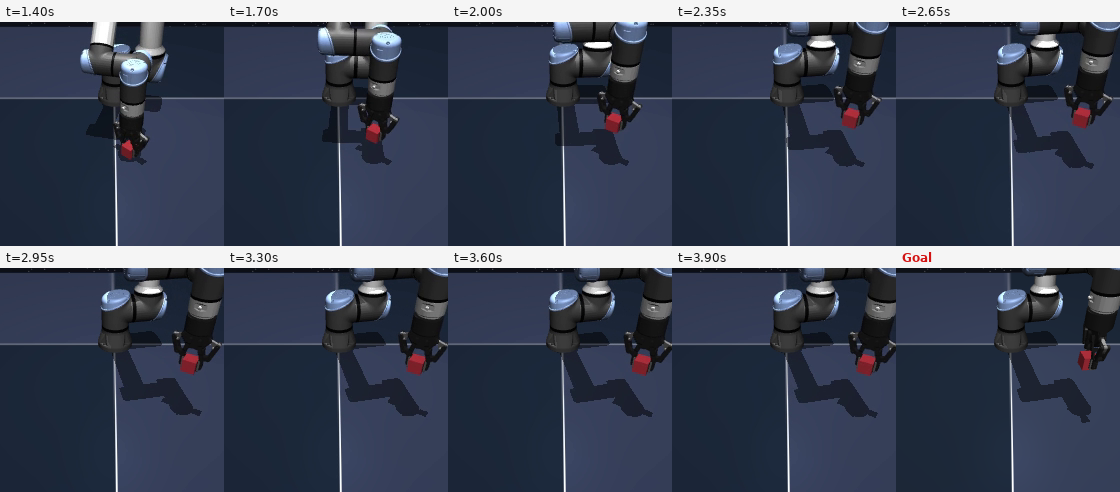}
    \vspace{-1.5em}
    \caption{\textbf{Cube.} iLQR-based long-horizon planning. The planner successfully moves the cube to the goal pose without dropping it, unlike the baseline approaches.}
    \label{fig:cube_nominal}
\end{figure}

\begin{figure}[H]
    \centering\vspace{-8pt}
    \includegraphics[width=\linewidth, trim={0.0cm 0.0cm 0.0cm 0.0cm}, clip]{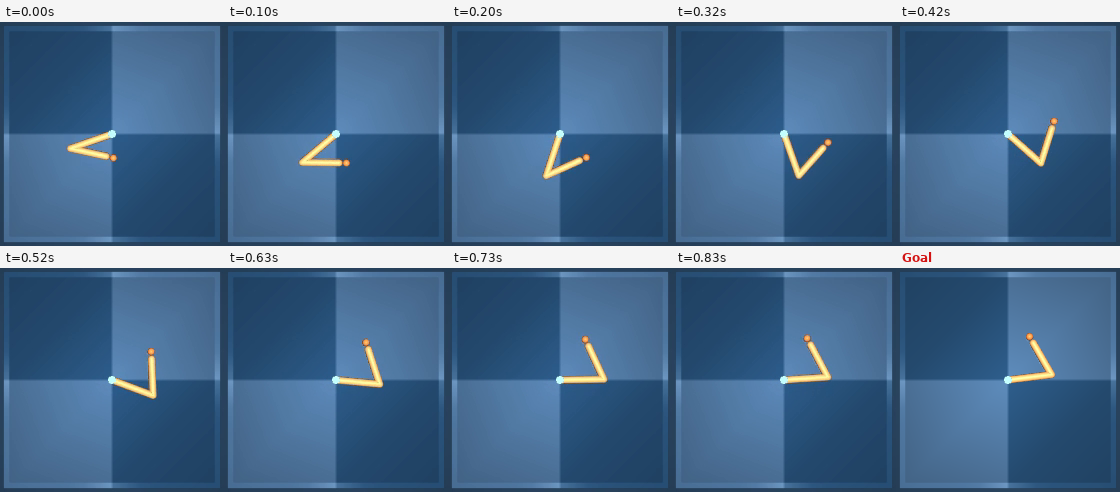}
    \vspace{-1.5em}
    \caption{\textbf{Reacher.} iLQR-based long-horizon planning. Similar to Figure \ref{fig:cube_nominal}, our gradient-based planner is able to successfully move to the goal in an efficient manner.}
    \label{fig:reacher_nominal}
\end{figure}

\begin{figure}[H]
    \centering\vspace{-8pt}
    \includegraphics[width=\linewidth, trim={0.0cm 0.0cm 0.0cm 0.0cm}, clip]{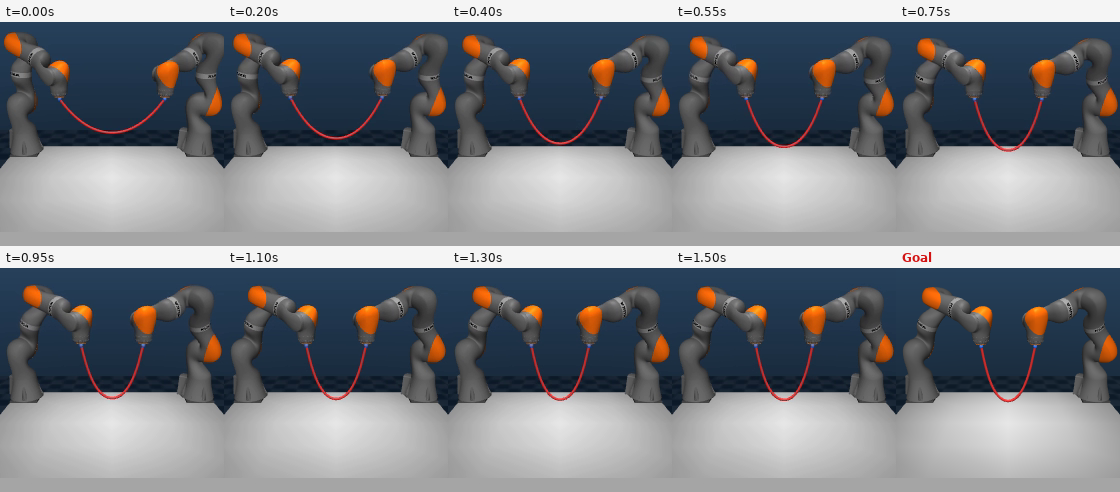}
    \vspace{-1.5em}
    \caption{\textbf{Rope.} iLQR-based long-horizon planning. The planner is able to correctly manipulate the deformable rope to reach the goal configuration.}
    \label{fig:rope_nominal}
\end{figure}

\begin{figure}[H]
    \centering\vspace{-8pt}
    \includegraphics[width=\linewidth, trim={0.0cm 0.0cm 0.0cm 0.0cm}, clip]{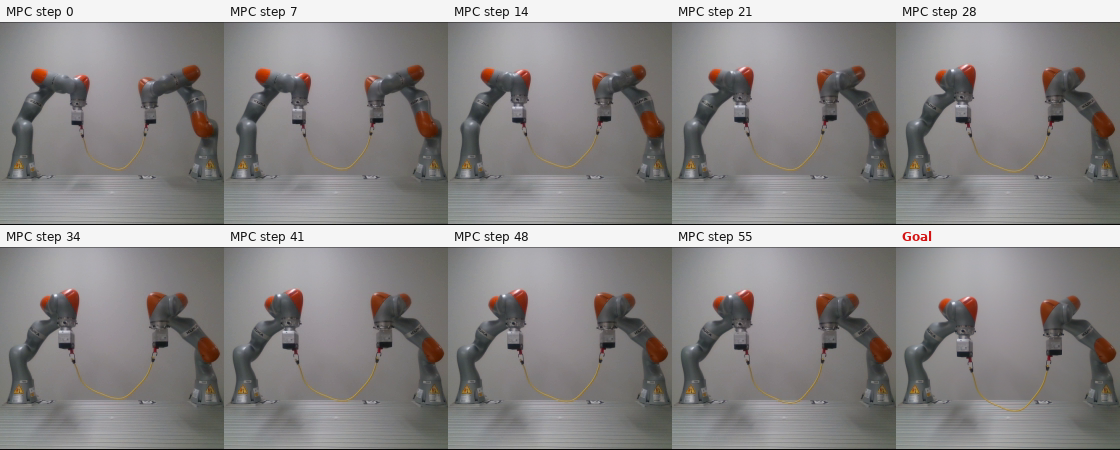}
    \vspace{-1.5em}
    \caption{\textbf{Real.} iLQR-based long-horizon planning on real hardware directly from pixels without state estimation. In this example, we demonstrate the arms moving forward to reach the goal image.}
    \label{fig:rope_real_nominal}
\end{figure}

\begin{figure}[H]
    \centering\vspace{-8pt}
    \includegraphics[width=\linewidth, trim={0.0cm 0.0cm 0.0cm 0.0cm}, clip]{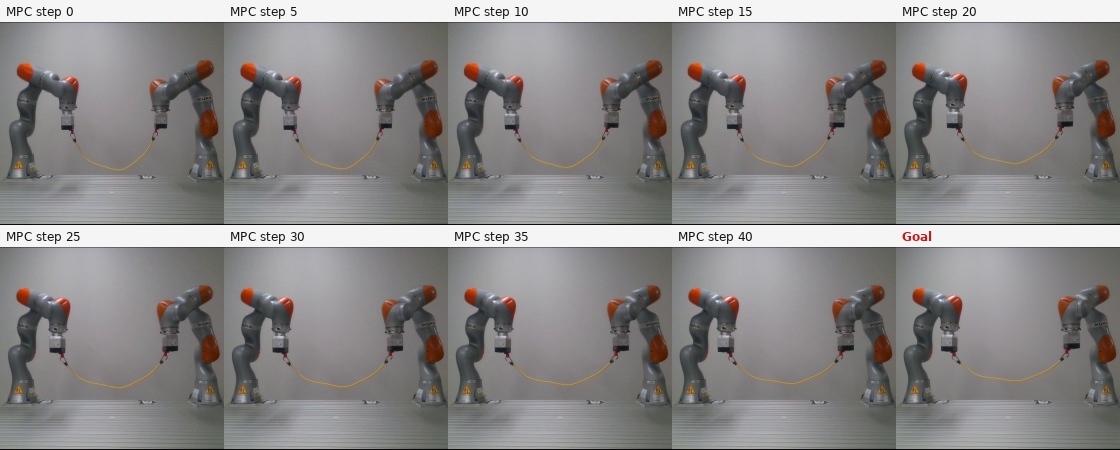}
    \vspace{-1.5em}
    \caption{\textbf{Real.} Another example of iLQR-based long-horizon planning on real hardware directly from pixels without state estimation. In this example, the bimanual setup has to stretch the rope to reach the goal image.}
    \label{fig:rope_real_2_nominal}
\end{figure}

\begin{figure}[H]
    \centering\vspace{-8pt}
    \includegraphics[width=\linewidth, trim={0.0cm 0.0cm 0.0cm 0.0cm}, clip]{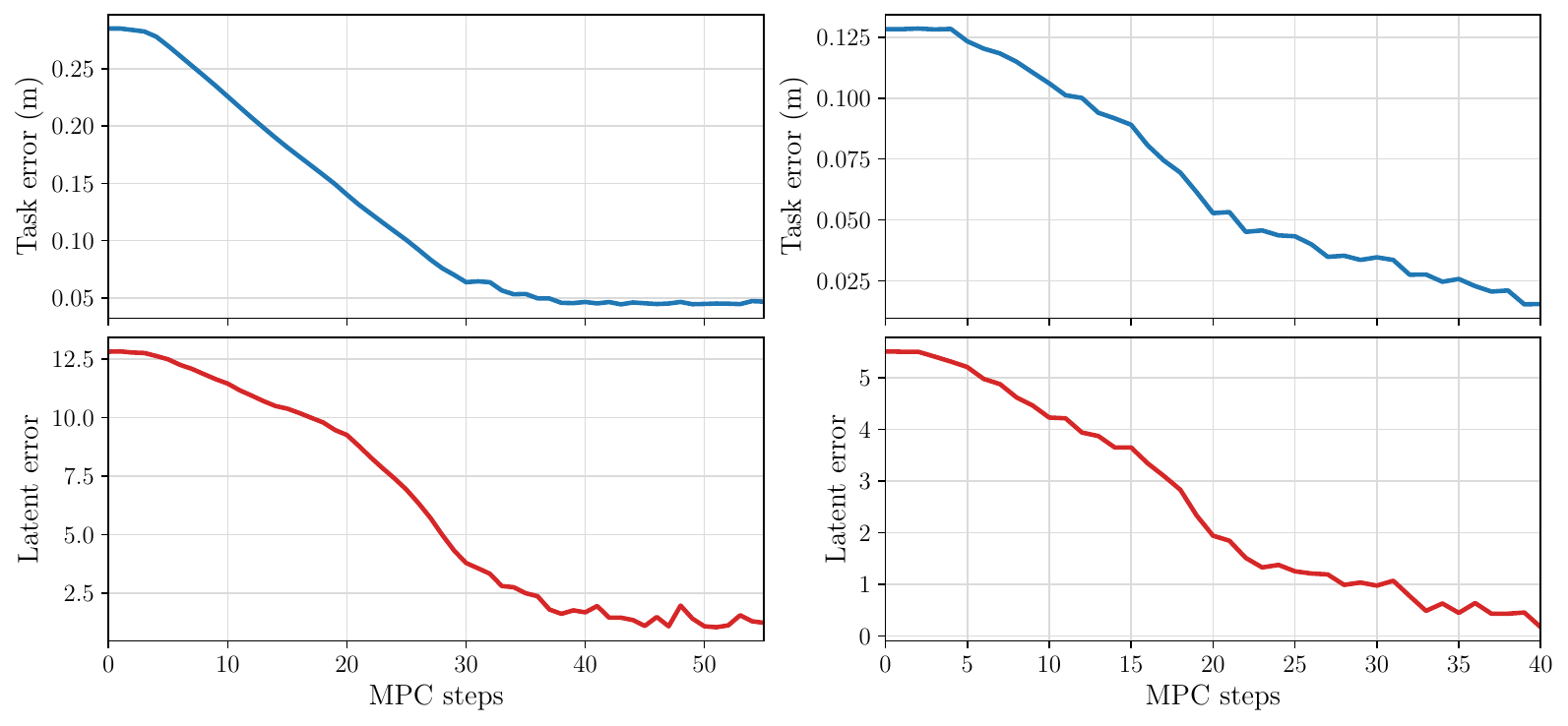}
    \vspace{-1.5em}
    \caption{\textbf{Real.} Task and latent space error plots demonstrating convergence to the desired goal state on hardware. Plot for Figure \ref{fig:rope_real_nominal} on the left and plot for Figure \ref{fig:rope_real_2_nominal} on the right.}
    \label{fig:rope_real_task_error}
\end{figure}

\subsection{Safe vs. Unsafe Time-lapses}
\begin{figure}[H]
    \centering\vspace{-8pt}
    \includegraphics[width=1.0\linewidth, trim={0.0cm 0.0cm 0.0cm 0.0cm}, clip]{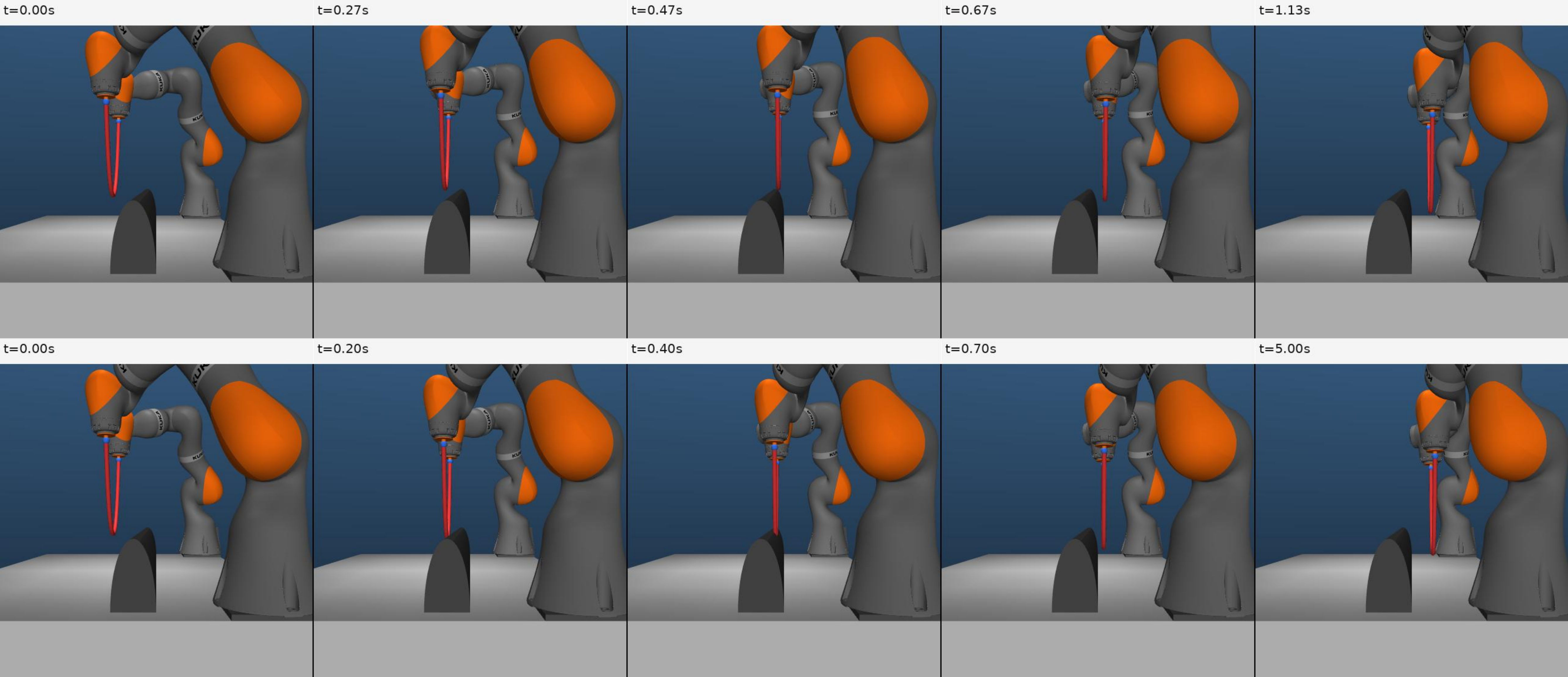}
    \vspace{-1.5em}
    \caption{\textbf{Rope.} Safe (top) plan vs. unsafe (bottom) plan time-lapse. The \method~planner successfully maneuvers the rope over the obstacle while reaching the goal image, unlike the nominal planner which collides with the obstacle.}
    \label{fig:rope_safe_unsafe}
\end{figure}
\vspace{-0.4cm}
\begin{figure}[H]
    \centering\vspace{-8pt}
    \includegraphics[width=1.0\linewidth, trim={0cm 0.0cm 0.0cm 0.0cm}, clip]{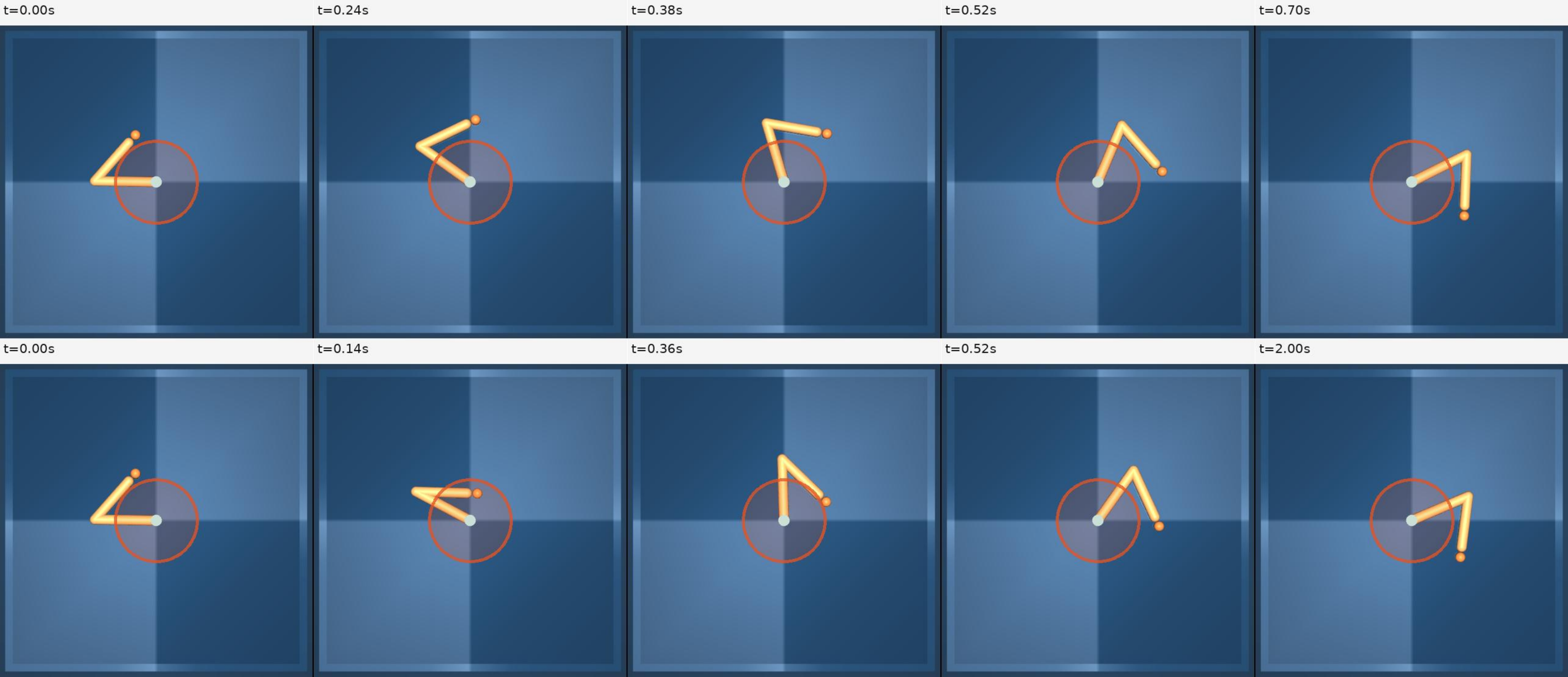}
    \vspace{-1.5em}
    \caption{\textbf{Reacher.} Safe (top) plan vs. unsafe (bottom) plan time-lapse. Similar to Figure \ref{fig:rope_safe_unsafe}, the nominal planner traverses through the ``unsafe" region and violates the pose constraint, while \method~avoids the unsafe region. }
    \label{fig:reacher_safe_unsafe}
\end{figure}
\vspace{-0.4cm}
\begin{figure}[H]
    \centering\vspace{-8pt}
    \includegraphics[width=1.0\linewidth, trim={0.0cm 0.0cm 0.0cm 0.0cm}, clip]{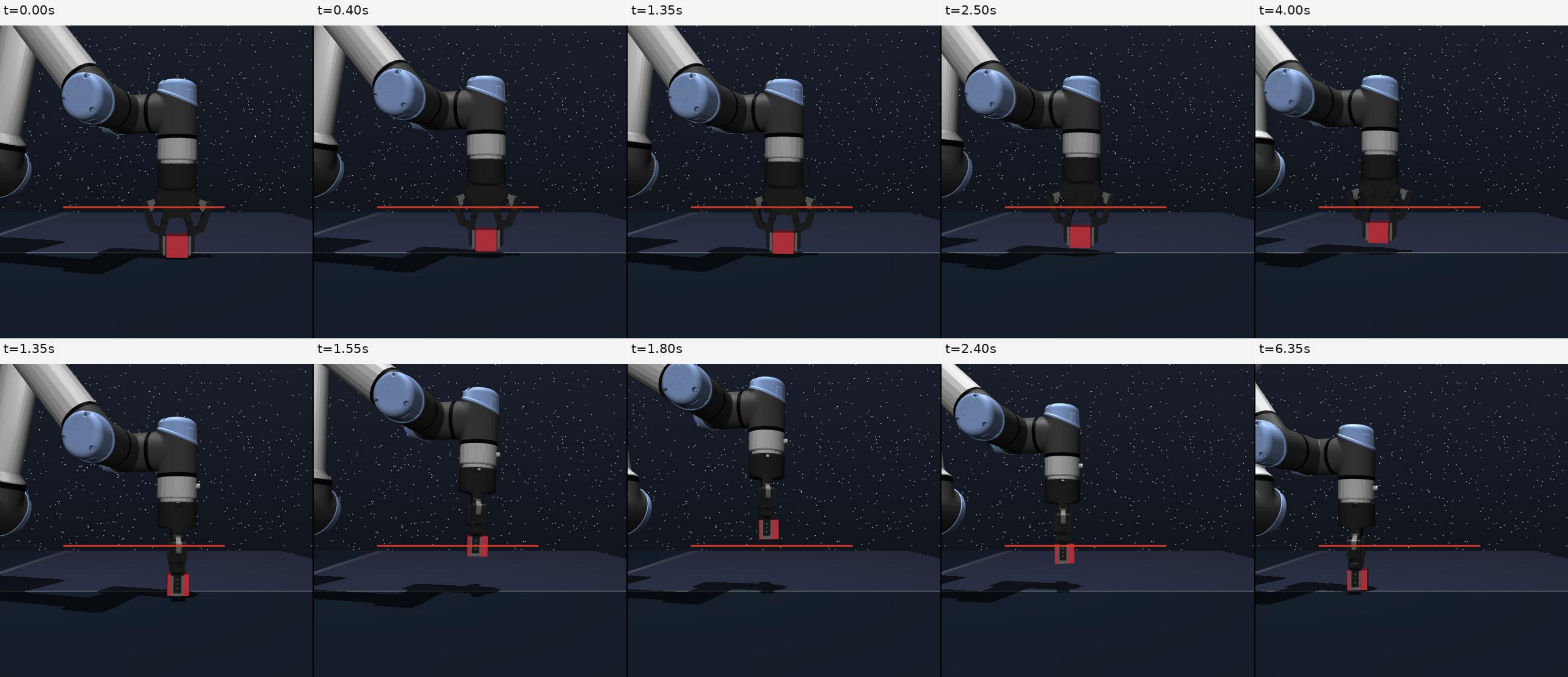}
    \vspace{-1.5em}
    \caption{\textbf{Cube.} Safe (top) plan vs. unsafe (bottom) plan time-lapse. \method~maintains the cube below the height threshold while reaching the goal.}
    \label{fig:cube_safe_unsafe}
\end{figure}

\subsection{Latent Reachability Tubes}
\begin{figure}[H]
    \centering\vspace{-8pt}
    \includegraphics[width=\linewidth, trim={0.0cm 0.0cm 0.0cm 0.0cm}, clip]{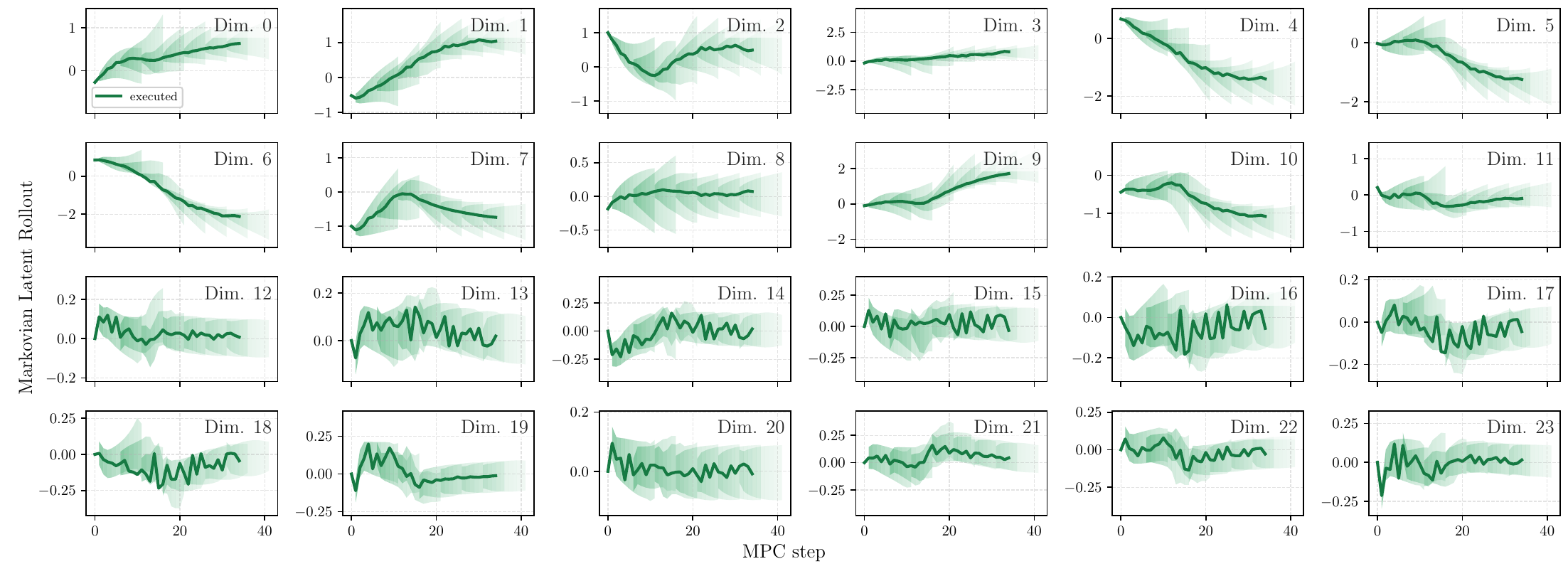}
    \vspace{-1.5em}
    \caption{\textbf{Rope.} Latent reachability tubes for the full Markov state computed via \method.}
    \label{fig:rope_tubes}
\end{figure}

\begin{figure}[H]
    \centering\vspace{-8pt}
    \includegraphics[width=\linewidth, trim={0.0cm 0.0cm 0.0cm 0.0cm}, clip]{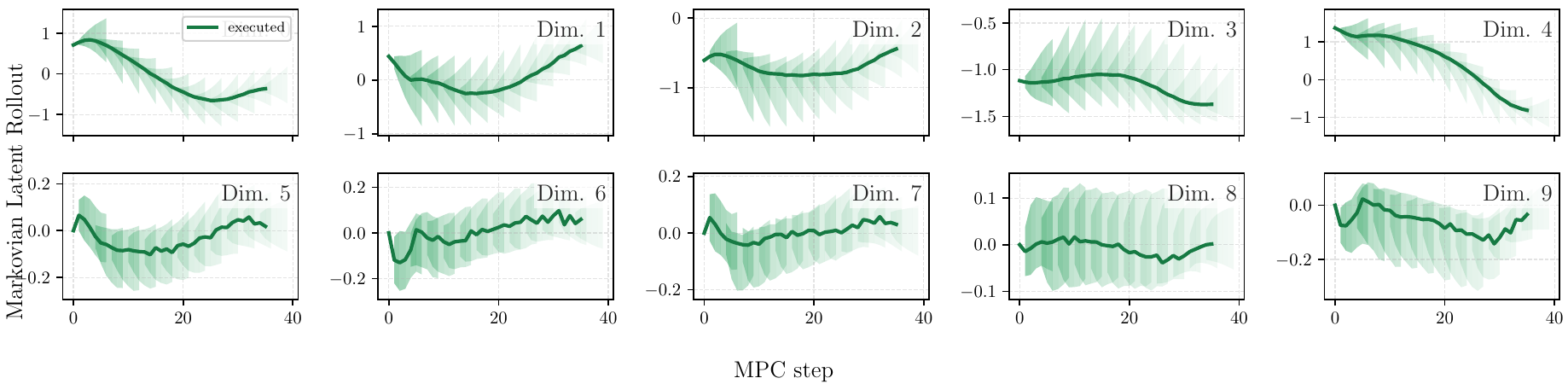}
    \vspace{-1.5em}
    \caption{\textbf{Reacher.} Latent reachability tubes for the full Markov state computed via \method.}
    \label{fig:reacher_tubes}
\end{figure}

\begin{figure}[H]
    \centering\vspace{-8pt}
    \includegraphics[width=\linewidth, trim={0.0cm 0.0cm 0.0cm 0.0cm}, clip]{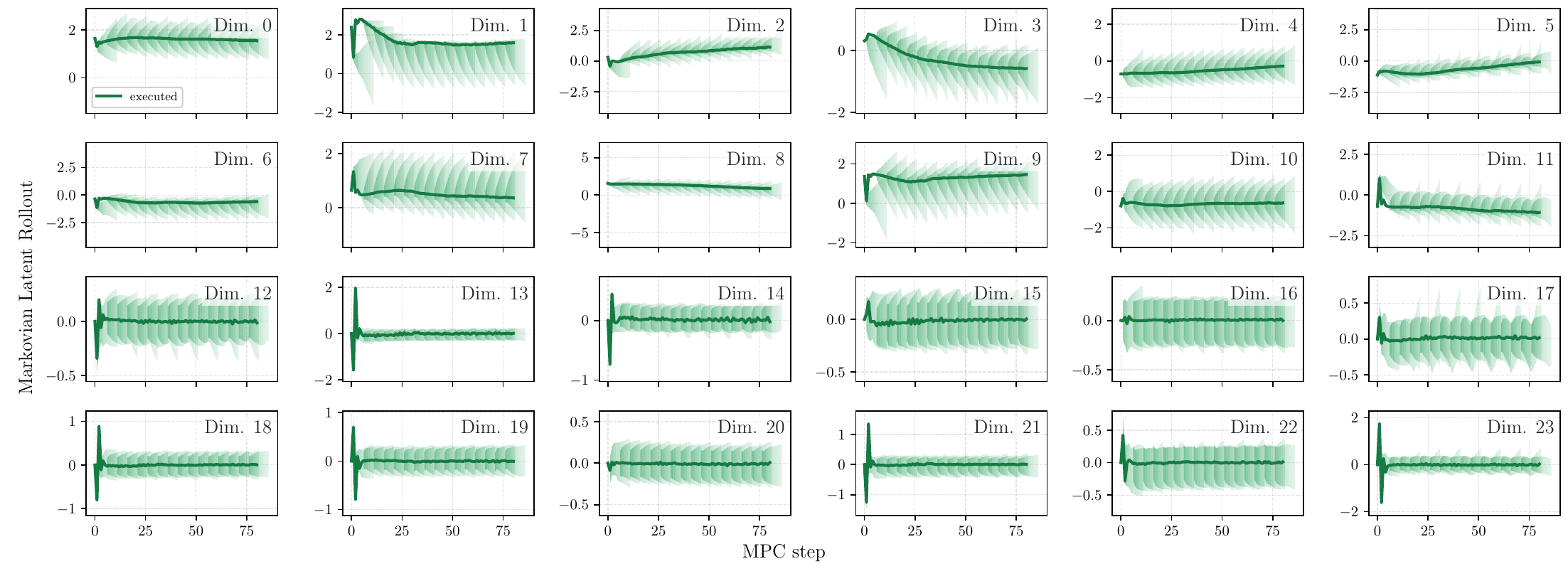}
    \vspace{-1.5em}
    \caption{\textbf{Cube.} Latent reachability tubes for the full Markov state computed via \method.}
    \label{fig:cube_tubes}
\end{figure}

\end{document}